\documentclass{article} %
\PassOptionsToPackage{table}{xcolor}
\usepackage[final]{colm2026_conference}

\usepackage[utf8]{inputenc}
\usepackage[T1]{fontenc}
\usepackage{microtype}
\usepackage{hyperref}
\usepackage{url}
\usepackage{booktabs}
\usepackage{amsfonts}
\usepackage{amsmath}
\usepackage{amssymb}
\usepackage{amsthm}
\usepackage{nicefrac}
\usepackage{xcolor}
\usepackage{graphicx}
\usepackage{multirow}
\usepackage{subcaption}
\usepackage{pifont}
\usepackage{wrapfig}
\usepackage{enumitem}
\usepackage{xspace}
\usepackage{algorithm}
\usepackage{algpseudocode}

\definecolor{darkblue}{rgb}{0, 0, 0.5}
\hypersetup{colorlinks=true, citecolor=darkblue, linkcolor=darkblue, urlcolor=darkblue}

\definecolor{lightblue}{RGB}{220,235,250}
\definecolor{oursblue}{RGB}{240,240,240}

\newtheorem{proposition}{Proposition}

\newcommand{\prism}{\textsc{Prism-$\Delta$}\xspace}
\newcommand{\prismk}{\textsc{Prism-$\Delta$}\xspace}
\newcommand{\prismkv}{\textsc{Prism-$\Delta$V}\xspace}
\newcommand{\seka}{SEKA\xspace}
\newcommand{\adaseka}{AdaSEKA\xspace}

\newcommand{\qwenA}{\textsc{Qwen3-4B}\xspace}
\newcommand{\qwenB}{\textsc{Qwen3-8B}\xspace}
\newcommand{\qwenC}{\textsc{Qwen3-14B}\xspace}
\newcommand{\gemmaA}{\textsc{Gemma3-4B}\xspace}
\newcommand{\gemmaB}{\textsc{Gemma3-12B}\xspace}

\title{\raisebox{-3pt}{\includegraphics[height=1.3em]{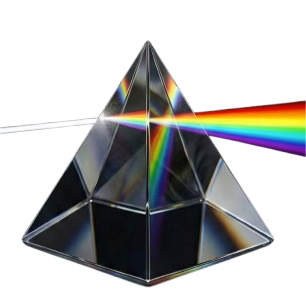}}~\textsc{Prism-$\Delta$}: Differential Subspace Steering for Prompt Highlighting in Large Language Models}

\author{%
  Yuyao Ge$^{\clubsuit}$ \quad
  Shenghua Liu$^{\clubsuit}$\thanks{Corresponding author.} \quad
  Yiwei Wang$^{\triangle}$ \quad
  Baolong Bi$^{\clubsuit}$ \\[3pt]
  \textbf{Lingrui Mei}$^{\clubsuit}$ \quad
  \textbf{Jiayu Yao}$^{\clubsuit}$ \quad
  \textbf{Jiafeng Guo}$^{\clubsuit}$ \quad
  \textbf{Xueqi Cheng}$^{\clubsuit}$ \\[6pt]
  $^{\clubsuit}$Institute of Computing Technology, Chinese Academy of Sciences \\
  $^{\triangle}$University of California, Merced \\
  \texttt{\{geyuyao24z, liushenghua\}@ict.ac.cn} \\[9pt]
  \raisebox{-2.5pt}{\includegraphics[height=1.1em]{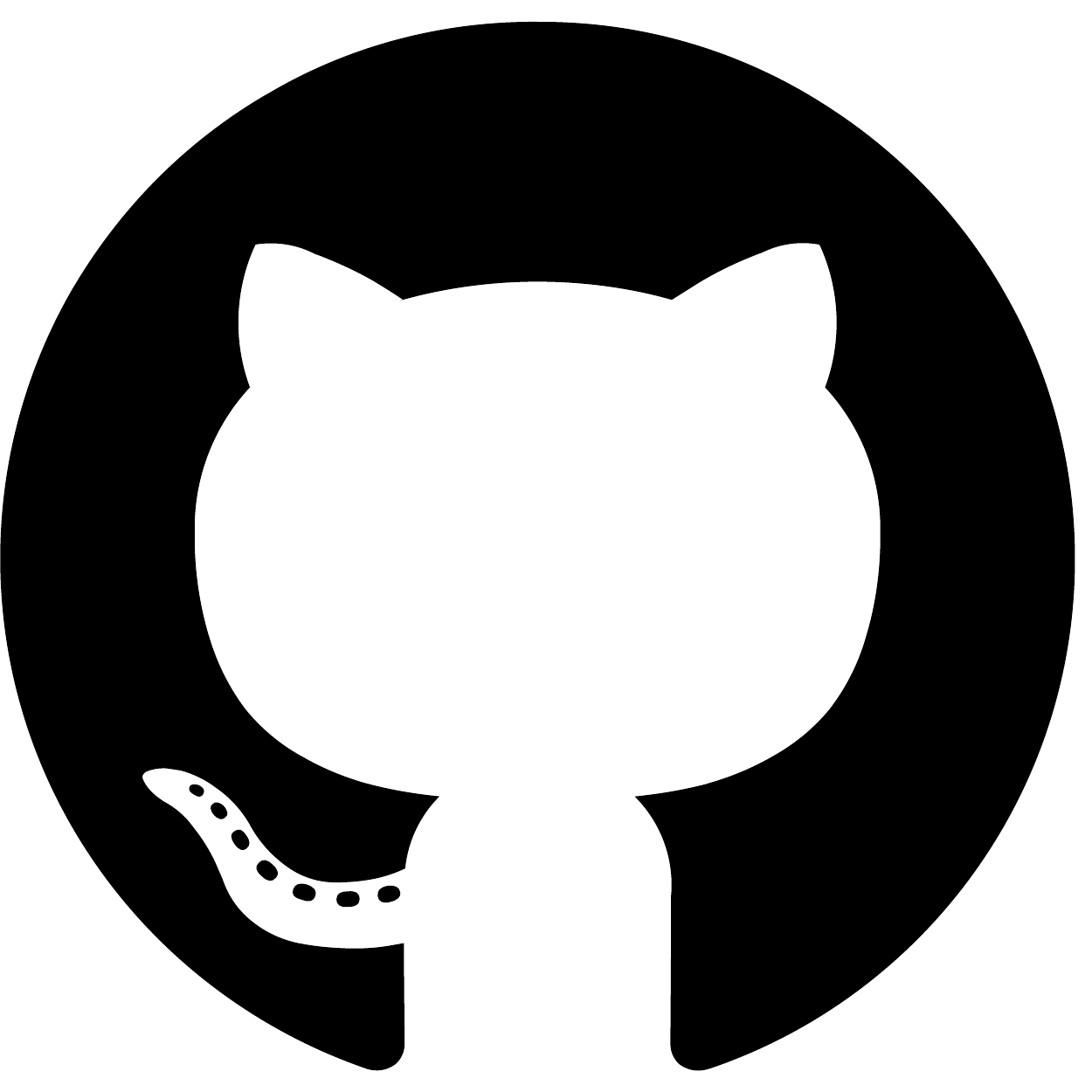}}~{\urlstyle{same}\url{https://github.com/YuyaoGe/PRISM-DELTA}}
}

\begin{document}

\maketitle
\lhead{\textsc{Prism-$\Delta$}: Differential Subspace Steering for Prompt Highlighting}

\begin{abstract}
Prompt highlighting steers a large language model to prioritize user-specified text spans during generation.
A key challenge of existing Key-editing approaches is extracting steering directions that capture the difference between relevant and irrelevant contexts, rather than shared structural patterns common to both.
We propose \prism{} (\textbf{P}rojection-based \textbf{R}elevance-\textbf{I}nformed \textbf{S}teering \textbf{M}ethod), which decomposes the difference between positive and negative cross-covariance matrices to maximize discriminative energy while eliminating shared directions. Each attention head receives a continuous softplus importance weight, letting weak-but-useful heads contribute at reduced strength.
The framework extends naturally to Value representations, capturing content-channel signal that Key-only methods leave unused.
Across four benchmarks and five models, our methods match or exceed the best existing method on 19 of 20 configurations, with relative gains up to +13.0\%, while halving the fluency cost of steering.
Our methods also scale to long-context retrieval, outperforming the best existing method by up to +5.1\% relative gain. \prism{} is compatible with FlashAttention and adds negligible memory overhead.
\end{abstract}

\section{Introduction}
\label{sec:intro}

Large language models (LLMs) frequently need to prioritize specific parts of their input. When presented with conflicting information, the model should attend to newly provided facts over its parametric memory. In long-context retrieval, the answer may reside in the middle of thirty passages, where models notoriously underperform~\citep{liu2024lost}. This problem is known as \emph{prompt highlighting}~\citep{li2026spectral}: given a prompt and a marked subset of tokens, the goal is to amplify the model's attention to those tokens so that generation becomes more accurate and faithful to user intent.

Several prompt highlighting methods have been proposed, including post-hoc attention score manipulation~\citep{pasta-zhang2024tell}, logit-level anchoring~\citep{spanchor}, and pre-attention Key vector editing via spectral decomposition~\citep{li2026spectral}. However, all of them operate solely on the \emph{routing channel}---Key representations that control where the model looks. Transformer attention output also depends on a \emph{content channel}---Value representations that determine what information is transmitted. Even when routing successfully directs the model to attend to highlighted tokens, the information those tokens convey through their Value representations remains unenhanced.

We investigate whether the Value channel carries useful signal by extracting Key and Value representations under contrastive conditions. Value shifts are comparable in magnitude to Key across all five tested models, with roughly half of all heads showing significant Value-channel signal. Key-only methods thus leave substantial signal unused.

We propose \prism{} (\textbf{P}rojection-based \textbf{R}elevance-\textbf{I}nformed \textbf{S}teering \textbf{M}ethod), which steers both channels through discriminative subspace learning and adaptive head weighting, as illustrated in Figure~\ref{fig:overview}. Our contributions are:

\begin{itemize}
    \item We introduce \emph{differential cross-covariance decomposition} as a provably optimal criterion for extracting discriminative steering directions (Proposition~\ref{prop:discriminative}): it maximizes the separation between relevant and irrelevant conditions while automatically eliminating shared structural directions that prior methods conflate with signal. Per-head softplus weighting enables continuous, rather than binary, importance assignment.
    \item We identify the \emph{content channel} (Value representations) as a fundamentally untapped source of steering signal, and introduce the first method to jointly steer both routing and content channels. Value steering reduces fluency degradation and can further improve accuracy.
\end{itemize}

\begin{figure}[t] %
    \centering
    \includegraphics[width=\textwidth]{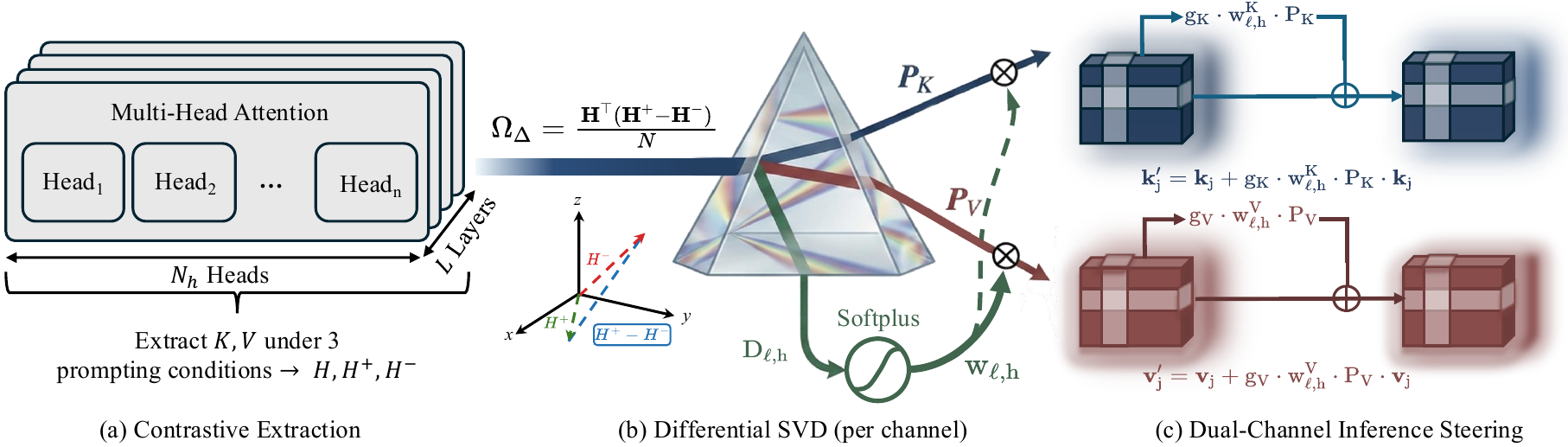}
    \caption{Overview of \prism{}. SVD decomposes $\Omega_\Delta$ into per-head projections ($P_K$, $P_V$) and importance weights ($w_{\ell,h}$), steering both Key and Value channels at inference.}
    \label{fig:overview}
\end{figure}

\section{Related work}
\label{sec:related}

\paragraph{Prompt highlighting.}
PASTA~\citep{pasta-zhang2024tell} modifies attention scores but is incompatible with FlashAttention~\citep{dao2022flashattention}, while SPA~\citep{spanchor} anchors at the logit level but requires multiple forward passes. \seka{}~\citep{li2026spectral} edits Key vectors via spectral decomposition with near-zero overhead. Prefix-Tuning~\citep{li-liang-2021-prefix} operates in Key/Value space by prepending learned soft tokens, but requires gradient-based training and modifies the context length. \citet{venkateswaran-contractor-2026-spotlight} dynamically steer attention toward user-specified prompt parts for instruction following, and \citet{zhan2026realreadingtransformeractivations} score head relevance with a vector-quantized autoencoder for head selection. These methods either operate only on the routing channel, require offline training, or lack principled discriminative subspace learning. \prism{} addresses all three gaps: it jointly steers Key and Value channels with per-head continuous softplus weighting, is gradient-free, and edits existing token representations in place. Orthogonally, retrieval-based and intent-aware context construction change which content reaches the model rather than how the model weighs it~\citep{fu2026contextnav,fu-etal-2026-videostir}; \prism{} operates downstream of any such selector, as Appendix~\ref{app:automatic_span} shows.

\paragraph{Activation editing and head specialization.}
Among activation editing approaches, SEA~\citep{qiu2024spectral} and ITI~\citep{li2023inferencetime} project activations along contrastive directions, Activation Addition~\citep{turner2023steering} and \citet{stolfo2025improving} steer via natural-language-derived or instruction-specific vectors, and Representation Engineering~\citep{zou2025representation}, \citet{subramani-etal-2022-extracting}, and CARVE~\citep{ge2025focusingcontrastiveattentionenhancing} extract steering representations at the population level---all operating on the residual stream without distinguishing Key and Value roles. In knowledge editing, ROME~\citep{meng2022locating}, AlphaEdit~\citep{fang2025alphaedit}, and \citet{hernandez2024inspecting} target factual associations via feed-forward weight updates or learned fact encodings, while attention analyses~\citep{clark-etal-2019-bert,voita-etal-2019-analyzing,NEURIPS2019_2c601ad9} and mechanistic studies on transformer circuits~\citep{elhage2021mathematical}, induction heads~\citep{olsson2022incontextlearninginductionheads}, and retrieval heads~\citep{wu2025retrieval} reveal head-level functional specialization that motivates per-head importance weighting. \prism{} uses differential SVD to extract discriminative directions and softplus to assign continuous weights, as detailed in Table~\ref{tab:method_comparison}.

\begin{table}[t]
\centering
\renewcommand{\arraystretch}{1}
\small
\resizebox{\textwidth}{!}{
\setlength{\tabcolsep}{10pt}
\begin{tabular}{lcccc}
\toprule
\textbf{Method} & \textbf{Steering target} & \textbf{Projection} & \textbf{Head selection} & \textbf{FlashAttn} \\
\midrule
Prefix-Tuning & Key + Value (prefix) & Gradient & None & \ding{51} \\
PASTA & Attn.\ matrix $\alpha$ & None & Profiling & \ding{55} \\
SPA & Logit distribution & None & None & \ding{51} \\
Spotlight & Attn.\ logits (bias) & Analytic & None & \ding{55} \\
REAL & Head activation & Codebook & Scoring top-$S$ & \ding{51} \\
\seka{} & Key & Indep.\ SVD & Hard threshold & \ding{51} \\
\adaseka{} & Key (multi-expert) & Indep.\ SVD & Hard threshold & \ding{51} \\
\midrule
\textbf{\prismk{}} & \textbf{Key} & \textbf{Diff.\ SVD} & \textbf{Softplus} & \ding{51} \\
\textbf{\prismkv{}} & \textbf{Key + Value} & \textbf{Diff.\ SVD} & \textbf{Softplus} & \ding{51} \\
\bottomrule
\end{tabular}
}
\caption{Comparison of prompt highlighting methods.}
\label{tab:method_comparison}
\end{table}

\section{Method}
\label{sec:method}

\subsection{Dual-channel view of attention}
\label{sec:dual_channel}

In standard multi-head attention, given input hidden states $\mathbf{x}_i$, each head $(\ell, h)$ computes:
\begin{equation}
\mathbf{q}_i = W_Q^{(\ell,h)} \mathbf{x}_i, \quad \mathbf{k}_j = W_K^{(\ell,h)} \mathbf{x}_j, \quad \mathbf{v}_j = W_V^{(\ell,h)} \mathbf{x}_j
\end{equation}
We drop the $(\ell,h)$ superscripts in subsequent equations for readability.
\begin{equation}
\text{output}_i = \sum_j \underbrace{\text{softmax}\!\left(\frac{\mathbf{q}_i^\top \mathbf{k}_j}{\sqrt{d}}\right)}_{\alpha_{ij}:\ \text{routing}} \cdot\ \underbrace{\mathbf{v}_j}_{\text{content}}
\label{eq:attention}
\end{equation}

\paragraph{Problem definition.} Given a prompt $\mathbf{x} = (x_1, \dots, x_T)$ and a set of highlighted tokens $\mathcal{S} \subset \{1, \dots, T\}$, the goal is to amplify the influence of tokens $j \in \mathcal{S}$ in the attention output.

\paragraph{Dual-channel decomposition.} The attention output is jointly determined by two functionally distinct channels: the \emph{routing channel} ($\mathbf{K} \to \alpha$: determines \emph{where} to look) and the \emph{content channel} ($\mathbf{V}$: determines \emph{what} information is transmitted). If we simultaneously perturb both channels for highlighted tokens:
\begin{align}
\text{output}'_i &= \sum_j (\alpha_{ij} + \Delta\alpha_{ij}) \cdot (\mathbf{v}_j + \Delta\mathbf{v}_j) \nonumber \\
&= \text{output}_i + \underbrace{\sum_j \Delta\alpha_{ij} \cdot \mathbf{v}_j}_{\text{routing gain}} + \underbrace{\sum_j \alpha_{ij} \cdot \Delta\mathbf{v}_j}_{\text{content gain}} + \underbrace{\sum_j \Delta\alpha_{ij} \cdot \Delta\mathbf{v}_j}_{\text{cross gain}}
\label{eq:decomposition}
\end{align}
Existing prompt highlighting methods capture only the routing gain; the content gain and cross gain remain unused. Figure~\ref{fig:dual_channel_evidence} confirms that Key and Value carry complementary signals, with their strength peaking at different network depths.

\begin{figure}[t] %
\centering
\begin{subfigure}[b]{0.48\linewidth}
    \includegraphics[width=\linewidth]{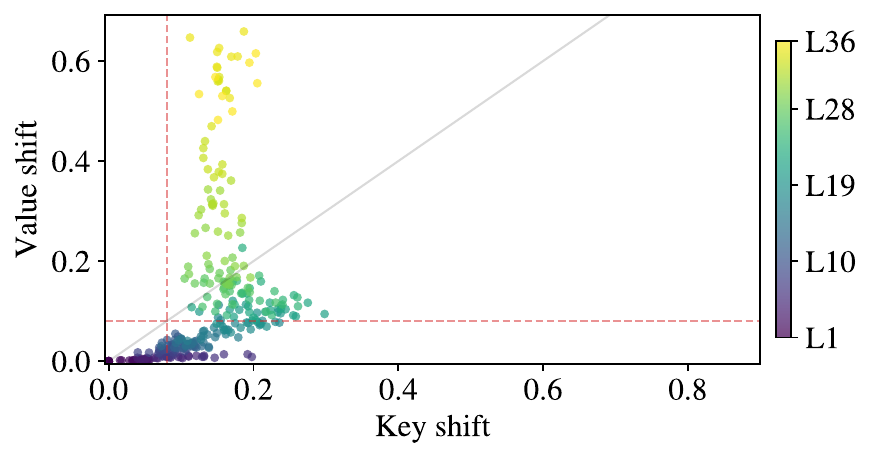}
    \caption{Per-head K vs.\ V discriminative shift.}
    \label{fig:kv_scatter}
\end{subfigure}
\hfill
\begin{subfigure}[b]{0.48\linewidth}
    \includegraphics[width=\linewidth]{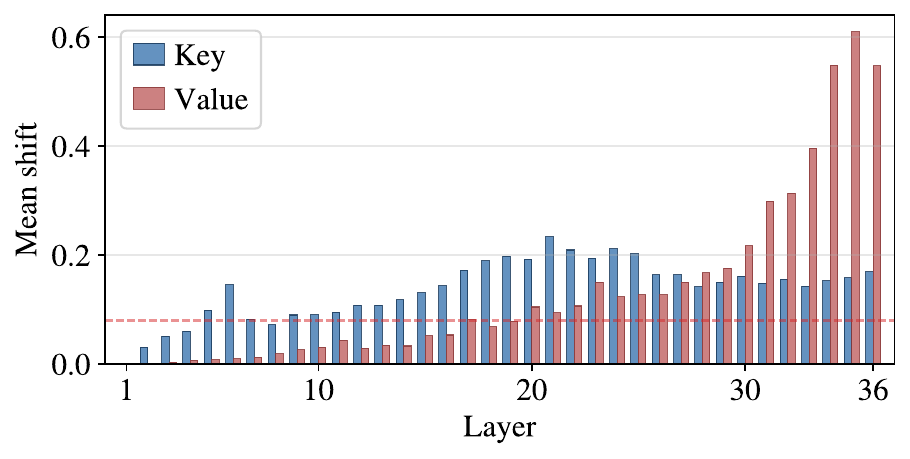}
    \caption{Layer-wise K and V signal strength.}
    \label{fig:layer_signal}
\end{subfigure}
\caption{Dual-channel discriminative signals in \qwenA{}-Base. \textbf{(a)} Each point is one attention head; Key and Value shifts are weakly correlated ($r{=}0.342$), confirming that the two channels carry complementary information. \textbf{(b)} Key signal peaks in middle layers, while Value signal peaks in late layers, suggesting functional specialization across depth.}
\label{fig:dual_channel_evidence}
\end{figure}

\subsection{Discriminative subspace learning}
\label{sec:subspace}

\paragraph{Contrastive data and representation extraction.} We construct synthetic QA triplets to identify contrastive directions. For each text context, we extract representations (Key or Value) at the answer token position under three conditions: $\mathbf{H}$ (neutral, context only), $\mathbf{H}^+$ (positive, context + relevant question), and $\mathbf{H}^-$ (negative, context + irrelevant question).

\paragraph{Differential cross-covariance.} We define the uncentered cross-covariance $\Omega^+ = \mathbf{H}^\top \mathbf{H}^+ / N$. Its top singular directions may include \emph{shared} directions that co-vary equally under both conditions. To isolate truly discriminative directions, we decompose the \emph{differential} cross-covariance:
\begin{equation}
\Omega_\Delta = \mathbf{H}^\top(\mathbf{H}^+ - \mathbf{H}^-) / N = \Omega^+ - \Omega^-
\label{eq:diff_cov}
\end{equation}

\begin{proposition}[Discriminative optimality of differential directions]
\label{prop:discriminative}
Let $\Omega_\Delta = U_\Delta \Sigma_\Delta V_\Delta^\top$ be the SVD of the differential cross-covariance. Then:
\begin{enumerate}
    \item[(a)] \textbf{Maximum discriminative energy.} The top-$k$ left singular vectors $\{u_1, \dots, u_k\}$ solve $\max_{U \in \mathbb{R}^{d \times k},\, U^\top U = I} \|U^\top \Omega_\Delta\|_F^2$, i.e., they capture the $k$-dimensional subspace that maximizes the cross-covariance difference between positive and negative conditions (by the Eckart--Young theorem).
    \item[(b)] \textbf{Automatic elimination of shared directions.} If a direction $\mathbf{u}_s$ satisfies $(\Omega^+)^\top \mathbf{u}_s = (\Omega^-)^\top \mathbf{u}_s$, then $\Omega_\Delta^\top \mathbf{u}_s = \mathbf{0}$: shared left-space directions contribute zero to the differential projection, regardless of their energy in $\Omega^+$.
\end{enumerate}
\end{proposition}

Proposition~\ref{prop:discriminative} requires no distributional assumptions; it holds for any finite sample set (proof in Appendix~\ref{app:proof}). Part~(a) directly justifies the choice of differential SVD over independent SVD, while part~(b) formally explains why shared variance directions are absent from the learned projection.

For comparison, let $P^+ = U^+_{:,:k}(U^+_{:,:k})^\top$ and $P^- = U^-_{:,:k}(U^-_{:,:k})^\top$ denote independent SVD projections of $\Omega^+$ and $\Omega^-$, using the same energy threshold $\gamma$ for rank selection. Figure~\ref{fig:proj_structure} visualizes the structure of these matrices for a representative head, showing that independent projections $P^+$ and $P^-$ share overlapping column spaces (structural redundancy), while $P_\Delta$ directly targets the differential subspace.

\begin{figure}[t] %
\centering
\begin{subfigure}[b]{0.33\textwidth}
    \includegraphics[width=\linewidth]{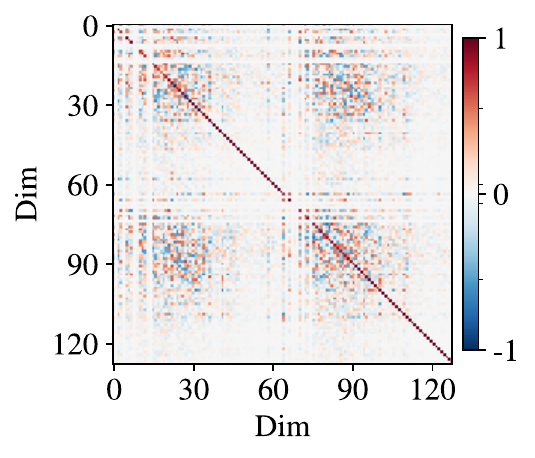}
    \caption{$P^+$ (independent, positive)}
    \label{fig:proj_pos}
\end{subfigure}%
\hspace{-2pt}%
\begin{subfigure}[b]{0.33\textwidth}
    \includegraphics[width=\linewidth]{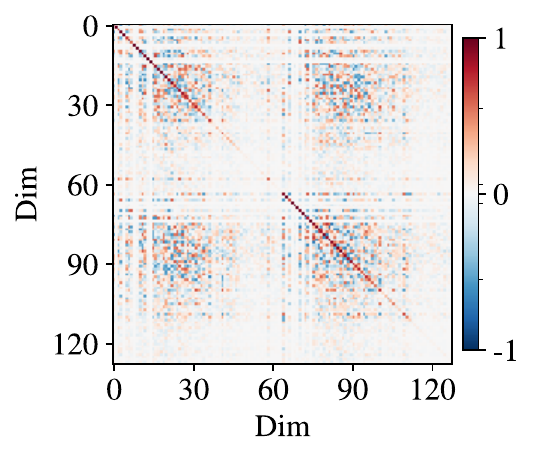}
    \caption{$P^-$ (independent, negative)}
    \label{fig:proj_neg}
\end{subfigure}%
\hspace{-2pt}%
\begin{subfigure}[b]{0.33\textwidth}
    \includegraphics[width=\linewidth]{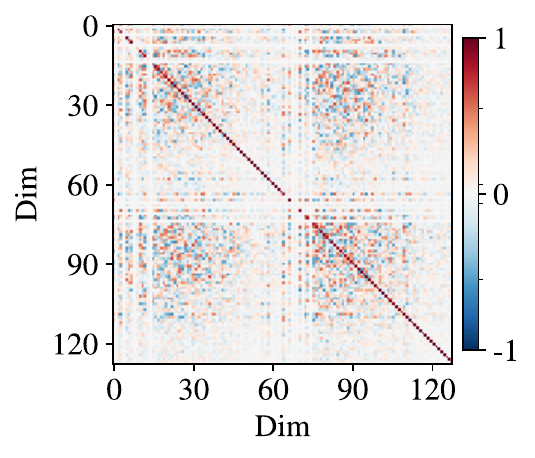}
    \caption{$P_\Delta$ (differential)}
    \label{fig:proj_delta}
\end{subfigure}
\caption{Projection matrix structure for layer 21, head 4 of \qwenA{} with $d{=}128$. Independent projections $P^+$ and $P^-$ exhibit overlapping subspaces, while the differential projection $P_\Delta$ directly targets the discriminative subspace.}
\label{fig:proj_structure}
\end{figure}

\begin{algorithm}[thb]
    \caption{\prism{}: offline projection learning and online steering}
    \label{alg:prism}
    \begin{algorithmic}[1]
    \State \textbf{Input:} Model $\mathcal{M}$ with $L$ layers and $n_h$ heads; contrastive triplets $\{(c_i,q_i^+,q_i^-)\}_{i=1}^N$; $\gamma,\delta_{\min},g_K,g_V$
    \For{each layer $\ell$ and head $h$}
        \State Extract $\mathbf{H},\mathbf{H}^+,\mathbf{H}^-\in\mathbb{R}^{N\times d}$ under neutral/positive/negative conditions
        \State $\Omega_\Delta^K\gets\mathbf{H}^{\top}(\mathbf{H}^+-\mathbf{H}^-)/N$; \quad $D_{\ell,h}^K\gets\frac{1}{N}\sum_i\|\mathbf{H}^+_{i,:}-\mathbf{H}^-_{i,:}\|_2$
        \State Compute SVD and retain top-$k$ vectors with cumulative energy $\geq\gamma$
        \State $P_K\gets U_{:,:k}U_{:,:k}^{\top}$; \quad $w_{\ell,h}^K\gets\mathrm{softplus}(D_{\ell,h}^K-\delta_{\min})$
        \State Repeat for Value representations to obtain $P_V,w_{\ell,h}^V$
    \EndFor
    \For{each highlighted token $j\in\mathcal{S}$ at each layer and head}
        \State $\mathbf{k}'_j\gets\mathbf{k}_j+g_Kw_{\ell,h}^KP_K\mathbf{k}_j$
        \State $\mathbf{v}'_j\gets\mathbf{v}_j+g_Vw_{\ell,h}^VP_V\mathbf{v}_j$
    \EndFor
    \State Proceed with $\mathrm{softmax}(QK'^\top/\sqrt{d})V'$
    \end{algorithmic}
    \end{algorithm}

The top-$k$ left singular vectors form the projection matrix:
\begin{equation}
P = U_\Delta[:,:k] \cdot U_\Delta[:,:k]^\top
\label{eq:projection}
\end{equation}
where $k$ is chosen so that the cumulative energy ratio reaches a threshold $\gamma$: $\sum_{i=1}^{k} \sigma_i^2 / \sum_{i=1}^{d} \sigma_i^2 \geq \gamma$.

\paragraph{Head importance weighting.} Different attention heads vary widely in their sensitivity to prompt highlighting. The extracted representations also yield a per-head discriminability measure---the norm difference $D_{\ell,h} = \frac{1}{N}\sum_i \|\mathbf{r}_i^+ - \mathbf{r}_i^-\|_2$ (where $\mathbf{r}$ denotes the head's Key or Value representation)---which quantifies how much the head's output shifts between positive and negative conditions. We map this to a continuous weight via the softplus function:
\begin{equation}
w_{\ell,h} = \text{softplus}(D_{\ell,h} - \delta_{\min}) = \log(1 + \exp(D_{\ell,h} - \delta_{\min}))
\label{eq:softplus}
\end{equation}
We choose softplus because it assigns larger weights to strongly discriminative heads and reduced but nonzero weights to weaker ones.

The projection $P$ and weight $w$ share the same data signal $D_{\ell,h}$: a large $D$ means the head's contrastive signal is strong, so its projection is reliable (Proposition~\ref{prop:discriminative}) and it should receive stronger steering. The two are not independent choices but consequences of a single data-driven measure.

\subsection{Dual-channel steering}
\label{sec:steering}

We apply the subspace learning framework (Section~\ref{sec:subspace}) separately to Key and Value spaces:
\begin{align}
\text{Key:}\quad &\Omega_\Delta^K = \mathbf{H}_K^\top(\mathbf{H}_K^+ - \mathbf{H}_K^-)/N \xrightarrow{\text{SVD}} P_K,\; w_K \\
\text{Value:}\quad &\Omega_\Delta^V = \mathbf{H}_V^\top(\mathbf{H}_V^+ - \mathbf{H}_V^-)/N \xrightarrow{\text{SVD}} P_V,\; w_V
\end{align}

At inference time, for each highlighted token $j \in \mathcal{S}$, both channels are simultaneously edited:
\begin{align}
\mathbf{k}_j' &= \mathbf{k}_j + g_K \cdot w_{\ell,h}^K \cdot P_K \cdot \mathbf{k}_j \label{eq:key_edit} \\
\mathbf{v}_j' &= \mathbf{v}_j + g_V \cdot w_{\ell,h}^V \cdot P_V \cdot \mathbf{v}_j \label{eq:val_edit}
\end{align}
where $g_K, g_V$ are gain scalars controlling the strength of routing and content steering, respectively.

\paragraph{Geometric interpretation.} Each transformation $(I + g \cdot w \cdot P) \cdot \mathbf{x}$ rescales the component of $\mathbf{x}$ within the learned subspace by a factor of $(1 + g \cdot w)$ while leaving the orthogonal component unchanged. Unlike uniform rescaling, each head's scaling factor is individually modulated by its discriminability $w_{\ell,h}$. This differs from adding a constant bias to the attention logit: Key scaling produces a \emph{query-dependent} boost---the attention shift varies with the semantic content of each query position---whereas logit bias applies a fixed offset regardless of query.

\paragraph{Instantiations.} The framework admits two variants:
\begin{itemize}
    \item \prismk{} ($g_V = 0$): steers the routing channel via pre-attention Key editing.
    \item \prismkv{} ($g_K > 0, g_V > 0$): steers both channels, capturing all three gain terms in Eq.~\ref{eq:decomposition}.
\end{itemize}
Algorithm~\ref{alg:prism} states the complete pipeline, which both variants share.

\section{Setup}
\label{sec:setup}

\paragraph{Benchmarks.} We evaluate on three prompt highlighting benchmarks: \textbf{BiasBios}~\citep{biasbios} (occupation prediction from highlighted biographies; metrics: Accuracy, Fluency, Consistency), \textbf{CounterFact}~\citep{meng2022locating} (knowledge conflict resolution; metrics: Efficacy, Paraphrase), and \textbf{Pronoun Change}~\citep{li2026spectral} (rewriting gendered pronouns to gender-neutral forms according to highlighted instructions; metrics: P.~Score, All-changed P.~Score). Formal metric definitions are in Appendix~\ref{app:metrics}; hyperparameters are tuned on held-out validation sets.

\paragraph{Models.} We evaluate across two architecture families and three scales: \qwenA{}/8B/14B-Base~\citep{yang2025qwen3} and \gemmaA{}/12B-PT~\citep{gemmateam2025gemma3}. Both families use Grouped Query Attention (GQA); the head index $(\ell, h)$ refers to KV heads, and \prism{} constructs one projection per KV head.

\paragraph{Baselines.} We compare against five baselines:
\textbf{Vanilla} (no steering),
\textbf{**-marked} (surrounding highlighted text with asterisks),
\textbf{PASTA}~\citep{pasta-zhang2024tell},
\textbf{SPA}~\citep{spanchor}, and
\textbf{\seka{}}~\citep{li2026spectral}.
A comparison with \adaseka{}, a higher-cost multi-expert variant, is provided in Appendix~\ref{app:adaseka}.

\paragraph{\prism{} configurations.}
\prismk{} uses Key-only steering ($g_V = 0$).
\prismkv{} steers both Key and Value ($g_V > 0$).
Projections are constructed offline from 100 synthetic contrastive QA pairs. Hyperparameters are listed in Appendix~\ref{app:hyperparams}. All experiments use greedy decoding on single NVIDIA H20 GPUs. Full setup details are in Appendix~\ref{app:exp_setup}.

\section{Results}
\label{sec:results}

\subsection{Main results}
\label{sec:main_results}

\begin{table}[t] %
\centering
\renewcommand{\arraystretch}{1}
\setlength{\aboverulesep}{2pt}
\setlength{\belowrulesep}{2pt}
\small
\resizebox{\textwidth}{!}{
\setlength{\tabcolsep}{1pt}
\begin{tabular}{ll ccccc cc}
\toprule
\multirow{2}{*}{\raisebox{-3pt}{\textbf{Model}}} & \multirow{2}{*}{\raisebox{-3pt}{\textbf{Metric}}} & \multicolumn{5}{c}{\textit{Baselines}} & \multicolumn{2}{c}{\textit{Ours}} \\
\cmidrule(lr){3-7} \cmidrule(lr){8-9}
& & \textbf{Vanilla} & \textbf{**-marked} & \textbf{PASTA} & \textbf{SPA} & \textbf{\seka{}} & \cellcolor{oursblue}\textbf{\prismk{}} & \cellcolor{oursblue}\textbf{\prismkv{}} \\
\midrule
 \multicolumn{9}{c}{\textbf{BiasBios}} \\
\midrule
\qwenA{} & \multirow{5}{*}{Accuracy} & 79.80 & 82.94$_{\scriptscriptstyle\blacktriangle 3.14}$ & 89.58$_{\scriptscriptstyle\blacktriangle 9.78}$ & 68.00$_{\scriptscriptstyle\blacktriangledown 11.80}$ & 90.92$_{\scriptscriptstyle\blacktriangle 11.12}$ & \cellcolor{oursblue}\textbf{92.38}$_{\scriptscriptstyle\blacktriangle 12.58}$ & \cellcolor{oursblue}\underline{92.36}$_{\scriptscriptstyle\blacktriangle 12.56}$ \\
\qwenB{} &  & 76.22 & 80.60$_{\scriptscriptstyle\blacktriangle 4.38}$ & 86.32$_{\scriptscriptstyle\blacktriangle 10.10}$ & 37.02$_{\scriptscriptstyle\blacktriangledown 39.20}$ & 88.74$_{\scriptscriptstyle\blacktriangle 12.52}$ & \cellcolor{oursblue}\textbf{89.62}$_{\scriptscriptstyle\blacktriangle 13.40}$ & \cellcolor{oursblue}\underline{89.12}$_{\scriptscriptstyle\blacktriangle 12.90}$ \\
\qwenC{} &  & 85.10 & 90.94$_{\scriptscriptstyle\blacktriangle 5.84}$ & 88.46$_{\scriptscriptstyle\blacktriangle 3.36}$ & 57.86$_{\scriptscriptstyle\blacktriangledown 27.24}$ & 90.28$_{\scriptscriptstyle\blacktriangle 5.18}$ & \cellcolor{oursblue}\textbf{91.68}$_{\scriptscriptstyle\blacktriangle 6.58}$ & \cellcolor{oursblue}\underline{91.20}$_{\scriptscriptstyle\blacktriangle 6.10}$ \\
\gemmaA{} &  & 89.90 & 91.00$_{\scriptscriptstyle\blacktriangle 1.10}$ & 82.58$_{\scriptscriptstyle\blacktriangledown 7.32}$ & 48.02$_{\scriptscriptstyle\blacktriangledown 41.88}$ & \underline{92.42}$_{\scriptscriptstyle\blacktriangle 2.52}$ & \cellcolor{oursblue}\textbf{92.90}$_{\scriptscriptstyle\blacktriangle 3.00}$ & \cellcolor{oursblue}91.64$_{\scriptscriptstyle\blacktriangle 1.74}$ \\
\gemmaB{} &  & 91.32 & 92.90$_{\scriptscriptstyle\blacktriangle 1.58}$ & \textbf{94.72}$_{\scriptscriptstyle\blacktriangle 3.40}$ & 46.88$_{\scriptscriptstyle\blacktriangledown 44.44}$ & \underline{93.04}$_{\scriptscriptstyle\blacktriangle 1.72}$ & \cellcolor{oursblue}92.22$_{\scriptscriptstyle\blacktriangle 0.90}$ & \cellcolor{oursblue}91.94$_{\scriptscriptstyle\blacktriangle 0.62}$ \\
\midrule
 \multicolumn{9}{c}{\textbf{CounterFact}} \\
\midrule
\multirow{2}{*}{\qwenA{}}
 & Efficacy & 45.00 & 57.70$_{\scriptscriptstyle\blacktriangle 12.70}$ & 97.16$_{\scriptscriptstyle\blacktriangle 52.16}$ & 65.24$_{\scriptscriptstyle\blacktriangle 20.24}$ & \underline{99.02}$_{\scriptscriptstyle\blacktriangle 54.02}$ & \cellcolor{oursblue}\textbf{99.14}$_{\scriptscriptstyle\blacktriangle 54.14}$ & \cellcolor{oursblue}98.08$_{\scriptscriptstyle\blacktriangle 53.08}$ \\
 & Paraphrase & 45.64 & 52.12$_{\scriptscriptstyle\blacktriangle 6.48}$ & 96.03$_{\scriptscriptstyle\blacktriangle 50.39}$ & 57.71$_{\scriptscriptstyle\blacktriangle 12.07}$ & \textbf{98.61}$_{\scriptscriptstyle\blacktriangle 52.97}$ & \cellcolor{oursblue}\underline{98.52}$_{\scriptscriptstyle\blacktriangle 52.88}$ & \cellcolor{oursblue}96.51$_{\scriptscriptstyle\blacktriangle 50.87}$ \\
\cmidrule(lr){2-2}
\multirow{2}{*}{\qwenB{}}
 & Efficacy & 39.04 & 56.24$_{\scriptscriptstyle\blacktriangle 17.20}$ & 92.70$_{\scriptscriptstyle\blacktriangle 53.66}$ & 69.26$_{\scriptscriptstyle\blacktriangle 30.22}$ & \underline{99.08}$_{\scriptscriptstyle\blacktriangle 60.04}$ & \cellcolor{oursblue}\textbf{99.24}$_{\scriptscriptstyle\blacktriangle 60.20}$ & \cellcolor{oursblue}98.46$_{\scriptscriptstyle\blacktriangle 59.42}$ \\
 & Paraphrase & 39.59 & 49.80$_{\scriptscriptstyle\blacktriangle 10.21}$ & 91.68$_{\scriptscriptstyle\blacktriangle 52.09}$ & 58.76$_{\scriptscriptstyle\blacktriangle 19.17}$ & \underline{98.96}$_{\scriptscriptstyle\blacktriangle 59.37}$ & \cellcolor{oursblue}\textbf{99.10}$_{\scriptscriptstyle\blacktriangle 59.51}$ & \cellcolor{oursblue}98.04$_{\scriptscriptstyle\blacktriangle 58.45}$ \\
\cmidrule(lr){2-2}
\multirow{2}{*}{\qwenC{}}
 & Efficacy & 37.56 & 45.52$_{\scriptscriptstyle\blacktriangle 7.96}$ & 76.84$_{\scriptscriptstyle\blacktriangle 39.28}$ & 84.22$_{\scriptscriptstyle\blacktriangle 46.66}$ & 98.92$_{\scriptscriptstyle\blacktriangle 61.36}$ & \cellcolor{oursblue}\textbf{99.00}$_{\scriptscriptstyle\blacktriangle 61.44}$ & \cellcolor{oursblue}\underline{98.98}$_{\scriptscriptstyle\blacktriangle 61.42}$ \\
 & Paraphrase & 36.12 & 40.12$_{\scriptscriptstyle\blacktriangle 4.00}$ & 66.33$_{\scriptscriptstyle\blacktriangle 30.21}$ & 76.11$_{\scriptscriptstyle\blacktriangle 39.99}$ & \textbf{99.02}$_{\scriptscriptstyle\blacktriangle 62.90}$ & \cellcolor{oursblue}\underline{98.82}$_{\scriptscriptstyle\blacktriangle 62.70}$ & \cellcolor{oursblue}98.72$_{\scriptscriptstyle\blacktriangle 62.60}$ \\
\cmidrule(lr){2-2}
\multirow{2}{*}{\gemmaA{}}
 & Efficacy & 55.04 & 57.56$_{\scriptscriptstyle\blacktriangle 2.52}$ & 78.36$_{\scriptscriptstyle\blacktriangle 23.32}$ & 93.90$_{\scriptscriptstyle\blacktriangle 38.86}$ & \underline{98.04}$_{\scriptscriptstyle\blacktriangle 43.00}$ & \cellcolor{oursblue}\textbf{98.10}$_{\scriptscriptstyle\blacktriangle 43.06}$ & \cellcolor{oursblue}97.68$_{\scriptscriptstyle\blacktriangle 42.64}$ \\
 & Paraphrase & 47.77 & 45.82$_{\scriptscriptstyle\blacktriangledown 1.95}$ & 59.53$_{\scriptscriptstyle\blacktriangle 11.76}$ & 91.92$_{\scriptscriptstyle\blacktriangle 44.15}$ & \textbf{98.83}$_{\scriptscriptstyle\blacktriangle 51.06}$ & \cellcolor{oursblue}\underline{96.24}$_{\scriptscriptstyle\blacktriangle 48.47}$ & \cellcolor{oursblue}96.01$_{\scriptscriptstyle\blacktriangle 48.24}$ \\
\cmidrule(lr){2-2}
\multirow{2}{*}{\gemmaB{}}
 & Efficacy & 45.34 & 48.72$_{\scriptscriptstyle\blacktriangle 3.38}$ & 68.30$_{\scriptscriptstyle\blacktriangle 22.96}$ & \underline{93.76}$_{\scriptscriptstyle\blacktriangle 48.42}$ & \textbf{98.86}$_{\scriptscriptstyle\blacktriangle 53.52}$ & \cellcolor{oursblue}\textbf{98.86}$_{\scriptscriptstyle\blacktriangle 53.52}$ & \cellcolor{oursblue}93.40$_{\scriptscriptstyle\blacktriangle 48.06}$ \\
 & Paraphrase & 37.21 & 36.69$_{\scriptscriptstyle\blacktriangledown 0.52}$ & 71.72$_{\scriptscriptstyle\blacktriangle 34.51}$ & 91.24$_{\scriptscriptstyle\blacktriangle 54.03}$ & \underline{99.27}$_{\scriptscriptstyle\blacktriangle 62.06}$ & \cellcolor{oursblue}\textbf{99.30}$_{\scriptscriptstyle\blacktriangle 62.09}$ & \cellcolor{oursblue}91.86$_{\scriptscriptstyle\blacktriangle 54.65}$ \\
\midrule
 \multicolumn{9}{c}{\textbf{Pronoun Change}} \\
\midrule
\multirow{2}{*}{\qwenA{}}
 & P.~Score & 93.14 & 95.76$_{\scriptscriptstyle\blacktriangle 2.62}$ & 95.82$_{\scriptscriptstyle\blacktriangle 2.68}$ & 80.27$_{\scriptscriptstyle\blacktriangledown 12.87}$ & 95.18$_{\scriptscriptstyle\blacktriangle 2.04}$ & \cellcolor{oursblue}\textbf{96.18}$_{\scriptscriptstyle\blacktriangle 3.04}$ & \cellcolor{oursblue}\underline{96.06}$_{\scriptscriptstyle\blacktriangle 2.92}$ \\
 & All-changed & 90.52 & 93.88$_{\scriptscriptstyle\blacktriangle 3.36}$ & \textbf{94.64}$_{\scriptscriptstyle\blacktriangle 4.12}$ & 78.19$_{\scriptscriptstyle\blacktriangledown 12.33}$ & 93.26$_{\scriptscriptstyle\blacktriangle 2.74}$ & \cellcolor{oursblue}\underline{94.60}$_{\scriptscriptstyle\blacktriangle 4.08}$ & \cellcolor{oursblue}94.52$_{\scriptscriptstyle\blacktriangle 4.00}$ \\
\cmidrule(lr){2-2}
\multirow{2}{*}{\qwenB{}}
 & P.~Score & 98.00 & 98.10$_{\scriptscriptstyle\blacktriangle 0.10}$ & 98.86$_{\scriptscriptstyle\blacktriangle 0.86}$ & 72.61$_{\scriptscriptstyle\blacktriangledown 25.39}$ & 98.56$_{\scriptscriptstyle\blacktriangle 0.56}$ & \cellcolor{oursblue}\underline{99.40}$_{\scriptscriptstyle\blacktriangle 1.40}$ & \cellcolor{oursblue}\textbf{99.48}$_{\scriptscriptstyle\blacktriangle 1.48}$ \\
 & All-changed & 97.84 & 97.84 & 98.72$_{\scriptscriptstyle\blacktriangle 0.88}$ & 74.95$_{\scriptscriptstyle\blacktriangledown 22.89}$ & 98.26$_{\scriptscriptstyle\blacktriangle 0.42}$ & \cellcolor{oursblue}\underline{99.30}$_{\scriptscriptstyle\blacktriangle 1.46}$ & \cellcolor{oursblue}\textbf{99.34}$_{\scriptscriptstyle\blacktriangle 1.50}$ \\
\cmidrule(lr){2-2}
\multirow{2}{*}{\qwenC{}}
 & P.~Score & 98.42 & \underline{98.86}$_{\scriptscriptstyle\blacktriangle 0.44}$ & 90.98$_{\scriptscriptstyle\blacktriangledown 7.44}$ & 91.60$_{\scriptscriptstyle\blacktriangledown 6.82}$ & 98.66$_{\scriptscriptstyle\blacktriangle 0.24}$ & \cellcolor{oursblue}\textbf{99.66}$_{\scriptscriptstyle\blacktriangle 1.24}$ & \cellcolor{oursblue}\textbf{99.66}$_{\scriptscriptstyle\blacktriangle 1.24}$ \\
 & All-changed & 98.22 & \underline{98.68}$_{\scriptscriptstyle\blacktriangle 0.46}$ & 90.94$_{\scriptscriptstyle\blacktriangledown 7.28}$ & 92.20$_{\scriptscriptstyle\blacktriangledown 6.02}$ & 98.54$_{\scriptscriptstyle\blacktriangle 0.32}$ & \cellcolor{oursblue}\textbf{99.56}$_{\scriptscriptstyle\blacktriangle 1.34}$ & \cellcolor{oursblue}\textbf{99.56}$_{\scriptscriptstyle\blacktriangle 1.34}$ \\
\cmidrule(lr){2-2}
\multirow{2}{*}{\gemmaA{}}
 & P.~Score & 41.34 & 38.86$_{\scriptscriptstyle\blacktriangledown 2.48}$ & 67.39$_{\scriptscriptstyle\blacktriangle 26.05}$ & 76.05$_{\scriptscriptstyle\blacktriangle 34.71}$ & 81.53$_{\scriptscriptstyle\blacktriangle 40.19}$ & \cellcolor{oursblue}\underline{89.08}$_{\scriptscriptstyle\blacktriangle 47.74}$ & \cellcolor{oursblue}\textbf{90.16}$_{\scriptscriptstyle\blacktriangle 48.82}$ \\
 & All-changed & 35.25 & 32.45$_{\scriptscriptstyle\blacktriangledown 2.80}$ & 66.43$_{\scriptscriptstyle\blacktriangle 31.18}$ & 74.45$_{\scriptscriptstyle\blacktriangle 39.20}$ & 81.11$_{\scriptscriptstyle\blacktriangle 45.86}$ & \cellcolor{oursblue}\underline{90.58}$_{\scriptscriptstyle\blacktriangle 55.33}$ & \cellcolor{oursblue}\textbf{91.68}$_{\scriptscriptstyle\blacktriangle 56.43}$ \\
\cmidrule(lr){2-2}
\multirow{2}{*}{\gemmaB{}}
 & P.~Score & 93.92 & 95.78$_{\scriptscriptstyle\blacktriangle 1.86}$ & 68.47$_{\scriptscriptstyle\blacktriangledown 25.45}$ & 86.41$_{\scriptscriptstyle\blacktriangledown 7.51}$ & \underline{97.70}$_{\scriptscriptstyle\blacktriangle 3.78}$ & \cellcolor{oursblue}\textbf{97.94}$_{\scriptscriptstyle\blacktriangle 4.02}$ & \cellcolor{oursblue}94.54$_{\scriptscriptstyle\blacktriangle 0.62}$ \\
 & All-changed & 94.96 & 96.42$_{\scriptscriptstyle\blacktriangle 1.46}$ & 68.01$_{\scriptscriptstyle\blacktriangledown 26.95}$ & 84.99$_{\scriptscriptstyle\blacktriangledown 9.97}$ & \underline{97.24}$_{\scriptscriptstyle\blacktriangle 2.28}$ & \cellcolor{oursblue}\textbf{97.46}$_{\scriptscriptstyle\blacktriangle 2.50}$ & \cellcolor{oursblue}94.42$_{\scriptscriptstyle\blacktriangledown 0.54}$ \\
\bottomrule
\end{tabular}
}
\caption{Main results across three benchmarks (\%). \textbf{Bold}: best; \underline{underline}: second best. Subscripts show absolute change vs.\ Vanilla.}
\label{tab:main_results}
\end{table}

\begin{table}[t] %
\centering
\renewcommand{\arraystretch}{0.9}
\small
\resizebox{\textwidth}{!}{
\setlength{\tabcolsep}{4pt}
\begin{tabular}{llllcccc}
\toprule
\multirow{2}{*}{\raisebox{-3pt}{\textbf{Configuration}}} & \multirow{2}{*}{\raisebox{-3pt}{\textbf{Projection}}} & \multirow{2}{*}{\raisebox{-3pt}{\textbf{Weighting}}} & \multirow{2}{*}{\raisebox{-3pt}{\textbf{Channels}}} & \multicolumn{2}{c}{\textbf{\qwenA{}}} & \multicolumn{2}{c}{\textbf{\qwenB{}}} \\
\cmidrule(lr){5-6} \cmidrule(lr){7-8}
& & & & \textbf{Acc} & \textbf{$\Delta$} & \textbf{Acc} & \textbf{$\Delta$} \\
\midrule
\prismkv{} & Diff.\ $\Omega_\Delta$ & Softplus & K+V & 92.36 & {+15.7} & 89.12 & {+16.9} \\
\prismk{} ($-$V) & Diff.\ $\Omega_\Delta$ & Softplus & K only & \textbf{92.38} & {\textbf{+15.8}} & \textbf{89.62} & {\textbf{+17.6}} \\
\prism{}-V ($-$K) & Diff.\ $\Omega_\Delta$ & Softplus & V only & 82.44 & {+3.3} & 79.16 & {+3.9} \\
\cmidrule{1-8}
$-$Differential & Independent & Softplus & K only & 91.52 & {+14.7} & 88.90 & {+16.6} \\
$-$Softplus & Differential & Uniform $w{=}1$ & K only & 91.42 & {+14.6} & 88.22 & {+15.7} \\
$-$Both & Independent & Uniform $w{=}1$ & K only & 91.44 & {+14.6} & 88.50 & {+16.1} \\
\midrule
\multicolumn{8}{l}{\textit{Baselines}} \\
\seka{} ($\delta{=}0.08$) & Dual indep. & Binary & K only & 89.56 & {+12.2} & 85.52 & {+12.2} \\
\seka{} (default) & Dual indep. & Binary $\delta{=}0.12$ & K only & 90.92 & {+13.9} & 88.74 & {+16.4} \\
Vanilla & --- & --- & --- & 79.80 & --- & 76.22 & --- \\
\bottomrule
\end{tabular}
}
\caption{Ablation study on BiasBios (Top-1 Accuracy \%). $\Delta$: relative improvement (\%) over Vanilla.}
\label{tab:ablation}
\end{table}

\begin{table}[t] %
\centering
\renewcommand{\arraystretch}{0.9}
\small
\resizebox{\textwidth}{!}{
\setlength{\tabcolsep}{16pt}
\begin{tabular}{lcccc}
\toprule
\textbf{Model} & \textbf{Vanilla} & \textbf{\seka{}} & \textbf{\prismk{}} & \textbf{\prismkv{}} \\
\midrule
\qwenA{} & 45.47 & 49.14$_{\scriptscriptstyle\blacktriangle 3.67}$ & \textbf{51.43}$_{\scriptscriptstyle\blacktriangle 5.96}$ & 50.57$_{\scriptscriptstyle\blacktriangle 5.10}$ \\
\qwenB{} & 46.69 & 50.29$_{\scriptscriptstyle\blacktriangle 3.60}$ & \textbf{52.71}$_{\scriptscriptstyle\blacktriangle 6.02}$ & 51.86$_{\scriptscriptstyle\blacktriangle 5.17}$ \\
\qwenC{} & 57.13 & 60.86$_{\scriptscriptstyle\blacktriangle 3.73}$ & \textbf{62.57}$_{\scriptscriptstyle\blacktriangle 5.44}$ & 62.00$_{\scriptscriptstyle\blacktriangle 4.87}$ \\
\gemmaA{} & 44.36 & 47.71$_{\scriptscriptstyle\blacktriangle 3.35}$ & 49.29$_{\scriptscriptstyle\blacktriangle 4.93}$ & \textbf{50.14}$_{\scriptscriptstyle\blacktriangle 5.78}$ \\
\gemmaB{} & 51.64 & \textbf{55.57}$_{\scriptscriptstyle\blacktriangle 3.93}$ & 54.43$_{\scriptscriptstyle\blacktriangle 2.79}$ & \textbf{55.57}$_{\scriptscriptstyle\blacktriangle 3.93}$ \\
\bottomrule
\end{tabular}
}
\caption{Lost-in-the-Middle: Average Exact Match (\%) across 7 gold positions. Steering targets the middle region (passages 4--25). Setup details in Appendix~\ref{app:litm_setup}.}
\label{tab:litm}
\end{table}

Table~\ref{tab:main_results} presents results across all three benchmarks and five models. On BiasBios, \prismk{} achieves the best accuracy on Qwen3 models with relative gains up to +1.6\%; on CounterFact, it ties \seka{} at 98.86\% on \gemmaB{} and reaches 99.24\% on \qwenB{}; on Pronoun Change, our methods outperform \seka{} on all 5 models with relative gains up to +10.6\% on P.~Score. Including the Lost-in-the-Middle results in Table~\ref{tab:litm}, our methods match or exceed the best existing method on 19 out of 20 model$\times$benchmark configurations; the exception is \gemmaB{} on BiasBios. Appendix~\ref{app:instruction_tuned} repeats the comparison on three instruction-tuned checkpoints.

\textbf{Statistical reliability.} Statistical reliability is high: across five projection subsets, standard deviation is 0.05--0.15\%, and \prismk{} matches or exceeds \seka{} on 14 out of 15 model$\times$benchmark cells ($p < 0.001$). Detailed variance analysis is in Appendix~\ref{app:seed_sensitivity} and sign test derivations in Appendix~\ref{app:statistics}.

\textbf{\textsc{Prism-$\Delta$} vs.\ \textsc{Prism-$\Delta$V}.} On Pronoun Change, \prismkv{} surpasses \prismk{} on \gemmaA{} by 1.08 percentage points, demonstrating that dual-channel steering can also improve accuracy when content rewriting is central. On \qwenA{} BiasBios the two variants are within 0.02 points. We recommend \prismk{} as the default and \prismkv{} when generation quality is prioritized; see Section~\ref{sec:dual_decomposition} for a detailed decomposition.

\subsection{Ablation study}
\label{sec:ablation}

Table~\ref{tab:ablation} uses a $2{\times}2$ factorial design. Removing softplus costs 0.96\% and removing differential projection costs 0.86\% on \qwenA{}. The interaction is strongly super-additive: from the $-$Both baseline, differential projection alone yields $-$0.02\% while softplus alone yields +0.08\%, yet their combination yields +0.94\%. At matched $\delta_{\min}{=}0.08$, \prismk{} achieves its best result while \seka{} drops by 1.36\%, because softplus smoothly down-weights noisy heads rather than giving all activated heads equal weight. \prism{}-V alone reaches 82.44\%, confirming that the Value channel carries independently useful discriminative signal. A CounterFact ablation (Appendix~\ref{app:cf_ablation}) and data quantity ablation (Appendix~\ref{app:data_ablation}) corroborate these findings.

\subsection{Lost-in-the-Middle}
\label{sec:litm}

Table~\ref{tab:litm} evaluates on the lost-in-the-middle benchmark~\citep{liu2024lost}, with detailed setup in Appendix~\ref{app:litm_setup}. Steering is applied selectively to the middle region of 30-passage contexts. Our methods match or exceed \seka{} on all five models, with relative gains up to 5.1\% (\prismkv{} on \gemmaA{}), confirming that \prism{} scales to long-context retrieval.

\begin{figure}[t] %
    \centering
    \includegraphics[width=\textwidth]{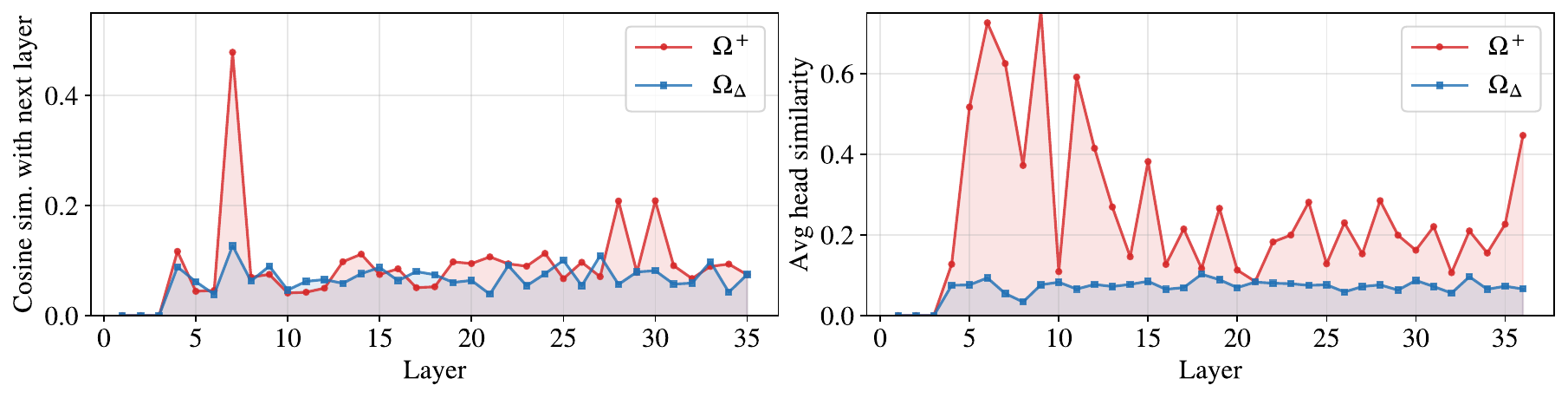}
    \caption{Direction consistency analysis. $\Omega^+$ directions show high cross-head similarity (dominated by shared structural directions), while $\Omega_\Delta$ directions are nearly independent (close to random baseline), confirming that differential projection extracts head-specific discriminative directions.}
    \label{fig:direction_consistency}
\end{figure}

\subsection{Efficiency analysis}
\label{sec:efficiency}

\begin{table}[t] \vspace{-0cm}
\centering
\renewcommand{\arraystretch}{0.9}
\small
\resizebox{\textwidth}{!}{
\setlength{\tabcolsep}{14pt}
\begin{tabular}{lcccc}
\toprule
\textbf{Method} & \textbf{Avg.~Latency (s)} & \textbf{Relative} & \textbf{Peak Mem (GB)} & \textbf{FlashAttn} \\
\midrule
Vanilla & 1.180 & 1.0$\times$ & 26.39 & \ding{51} \\
\seka{} & 1.194 \textcolor{gray}{(+0.01)} & 1.0$\times$ & 26.43 \textcolor{gray}{(+0.04)} & \ding{51} \\
\cmidrule{1-5}
\rowcolor{oursblue} \prismk{} & 1.481 \textcolor{gray}{(+0.30)} & 1.26$\times$ & 26.41 \textcolor{gray}{(+0.02)} & \ding{51} \\
\rowcolor{oursblue} \prismkv{} & 1.502 \textcolor{gray}{(+0.32)} & 1.27$\times$ & 26.43 \textcolor{gray}{(+0.04)} & \ding{51} \\
\bottomrule
\end{tabular}
}
\caption{Inference efficiency on \qwenB{}.}
\label{tab:efficiency}
\end{table}

As shown in Table~\ref{tab:efficiency}, \seka{} adds negligible overhead (+0.01\,s, +0.04\,GB) over Vanilla. \prismk{} adds +0.30\,s latency (1.26$\times$) due to the per-head weighted projection, with negligible additional memory (+0.02\,GB). \prismkv{} adds a further +0.02\,s for Value editing. Both variants remain fully compatible with FlashAttention~\citep{dao2022flashattention, dao2024flashattention2}. For comparison, PASTA adds +1.03\,s and +23.12\,GB, and SPA adds +5.32\,s~\citep{li2026spectral}, placing \prism{} in between \seka{} and these heavier alternatives.

\section{Analysis}
\label{sec:analysis}

\subsection{Discriminative subspace quality}
\label{sec:subspace_quality}

\begin{table}[b] %
\centering
\renewcommand{\arraystretch}{0.9}
\small
\resizebox{\textwidth}{!}{
\setlength{\tabcolsep}{6pt}
\begin{tabular}{lcccc}
\toprule
\textbf{Configuration} & \textbf{Top-1 Acc (\%)} & \textbf{Fluency ($\uparrow$)} & \textbf{Consistency ($\uparrow$)} & \textbf{Fluency Cost ($\downarrow$)} \\
\midrule
Vanilla & 79.80 & 4.638 & 0.116 & 0.000 \\
\seka{} & 90.92$_{\scriptscriptstyle\blacktriangle 11.12}$ & 3.681 & 0.112 & 0.957 \\
\cmidrule{1-5}
\rowcolor{oursblue} \prismk{} & \textbf{92.38}$_{\scriptscriptstyle\blacktriangle 12.58}$ & 4.134 & \textbf{0.127} & 0.504 \\
\rowcolor{oursblue} \prism{}-V & 82.44$_{\scriptscriptstyle\blacktriangle 2.64}$ & \textbf{4.666} & 0.121 & \textbf{$-$0.028} \\
\rowcolor{oursblue} \prismkv{} ($g_V{=}0.10$) & 92.36$_{\scriptscriptstyle\blacktriangle 12.56}$ & 4.116 & 0.125 & 0.522 \\
\bottomrule
\end{tabular}
}
\caption{Dual-channel decomposition on BiasBios with \qwenA{}. Best per metric in \textbf{bold}. Fluency Cost = Vanilla Fluency $-$ Method Fluency.}
\label{tab:dual_channel}
\end{table}

We validate the theoretical motivation of differential projection (Proposition~\ref{prop:discriminative}) on \qwenA{}-Base.

\paragraph{Direction consistency.} Figure~\ref{fig:direction_consistency} tracks the similarity of top singular directions across layers and heads. \textbf{Left:} adjacent-layer similarity measures how much a head's direction changes from one layer to the next. $\Omega^+$ directions show occasional cross-layer alignment spikes (0.2--0.5), while $\Omega_\Delta$ directions remain low throughout ($\leq$0.13), indicating layer-specific discriminative structure. \textbf{Right:} within-layer head similarity measures redundancy among heads in the same layer. $\Omega^+$ heads converge to similar directions (up to 0.76 in early layers), while $\Omega_\Delta$ heads remain near-independent throughout all layers. Aggregate statistics are provided in Table~\ref{tab:direction} of Appendix~\ref{app:direction_consistency}.

\subsection{Dual-channel contribution decomposition}
\label{sec:dual_decomposition}

Table~\ref{tab:dual_channel} reveals that the Key channel primarily drives accuracy (+12.58 points over Vanilla), while the Value channel preserves generation quality. \prismk{} incurs only 53\% of \seka{}'s fluency cost. Both \prismk{} and \prismkv{} simultaneously outperform \seka{} on all three metrics; Appendix~\ref{app:qualitative} confirms the fluency pattern across models and tasks using automatic and human-validated qualitative evaluation.

\textbf{Layer-wise and cross-model complementarity.} Figure~\ref{fig:layer_signal} shows that the two channels have complementary depth profiles. Key signal averages 0.175 in middle layers and Value averages 0.307 in late layers. Across models, Qwen3 becomes increasingly Value-dominant with scale (K/V ratio 1.02$\rightarrow$0.58), whereas Gemma3 remains Key-dominant (ratio $>1.2$); full measurements appear in Appendix~\ref{app:kv_cross_model}. This architecture-dependent balance is consistent with \prismk{} being the safer default, while Value steering is most helpful for fluency and rewriting. On Pronoun Change with \gemmaA{}, \prismkv{} outperforms \prismk{} by 1.08 points, an interaction consistent with the cross gain in Eq.~\ref{eq:decomposition}.

\subsection{Head selection robustness}
\label{sec:robustness}

\begin{wrapfigure}{r}{0.42\textwidth}
    \vspace{-1.0em}
    \centering
    \includegraphics[width=\linewidth]{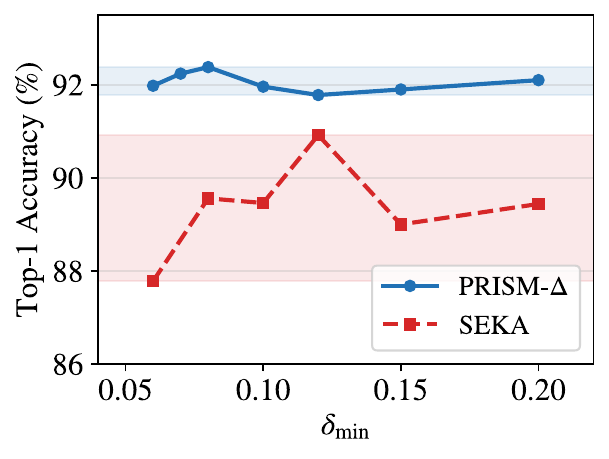}
    \caption{$\delta_{\min}$ sensitivity on BiasBios, \qwenA{}.}
    \label{fig:robustness}
    \vspace{-0.6em}
\end{wrapfigure}

Figure~\ref{fig:robustness} sweeps $\delta_{\min}$ from 0.06 to 0.20. \seka{} drops by 3.14\% from its optimum when $\delta_{\min}$ is reduced from 0.12 to 0.06: hard thresholding gives all activated heads equal weight $w{=}1$, so including noisy heads degrades performance. \prismk{} fluctuates by only 0.60\% across the same range, because softplus automatically assigns lower weights to less discriminative heads. This robustness substantially reduces tuning cost; hyperparameter sweeps are summarized in Appendix~\ref{app:sensitivity}. A per-sample log-probability analysis (Figure~\ref{fig:logprob}, Appendix~\ref{app:head_logprob}) compares the two methods sample by sample: \prismk{} rescues 154 samples that \seka{} misclassifies while losing only 81, for a net gain of +73. The 2:1 asymmetry reflects systematic correction of shared-feature confusion, consistent with the shared-direction elimination of Proposition~\ref{prop:discriminative}(b). Qualitative examples are in Appendix~\ref{app:case_study}. Appendices~\ref{app:automatic_span}, \ref{app:highlight_noise}, and~\ref{app:cross_domain} examine automatically selected spans, imperfect highlights, and cross-domain projection reuse.

\section{Conclusion}
\label{sec:conclusion}

We presented \prism{}, an inference-time prompt-highlighting method combining three ideas: differential cross-covariance decomposition removes directions shared by relevant and irrelevant contexts, continuous softplus weighting retains weak but useful heads, and dual-channel steering separates the routing role of Keys from the content role of Values. Key steering drives accuracy gains while Value steering reduces the fluency cost of editing.

\prism{} extends to instruction-tuned models and to spans supplied by automatic retrievers, making it a lightweight component for RAG and long-document pipelines. Future work can learn selector--steering policies jointly and extend differential subspace steering to broader instruction prioritization and multimodal contexts.

\section{Limitations}
\label{sec:limitations}

The main practical limitation is that the preferred gains vary substantially across benchmarks and models, so the current results require validation-set selection rather than a universal default. Gemma3 hyperparameters often differ from Qwen3 settings, and \gemmaB{} on Pronoun Change uses a negative $g_K$, a pattern also reported for \seka{}~\citep{li2026spectral}. On near-saturated benchmarks such as CounterFact, absolute differences between steering methods are small. The dual-channel variant has a clear failure case on \gemmaB{} CounterFact, which the available experiments do not explain. We therefore recommend \prismk{} as the default, with separate validation before enabling Value steering.

\newpage

\section*{Ethics Statement}
In conducting this research we have kept every component of the pipeline open and inspectable. All experiments use publicly available benchmarks~\citep{biasbios,meng2022locating,li2026spectral,liu2024lost} and publicly released model checkpoints~\citep{yang2025qwen3,gemmateam2025gemma3}, so that our findings rest on accessible and widely used resources. \prism{} operates at inference time and leaves model parameters unchanged: it rescales the representations of spans the user has already supplied rather than introducing new content. No personal data were collected, and the contrastive prompts used to build the projections are model-generated, as described in Appendix~\ref{app:llm_usage}.

\bibliography{references}
\bibliographystyle{colm2026_conference}

\appendix
\newpage
\section{Notation}
\label{app:notation}

Table~\ref{tab:notation} summarizes the notation used throughout the paper.

\begin{table}[htb]
\centering
\renewcommand{\arraystretch}{1.05}
\small
\setlength{\tabcolsep}{2pt}
\begin{tabular*}{\textwidth}{@{\extracolsep{\fill}}l l l}
\toprule
\textbf{Symbol} & \textbf{Definition} & \textbf{Description} \\
\midrule
\multicolumn{3}{l}{\textit{Model architecture}} \\
\midrule
$L$ & Number of layers & Total Transformer layers in the model \\
$n_h$ & Number of heads & KV heads per layer \\
$d$ & Head dimension & Dimensionality of each head's Key/Value space \\
$(\ell, h)$ & Layer-head index & Identifies a specific attention head \\
$\mathbf{x}_j$ & Token embedding & Input representation of token $j$ \\
$W_Q, W_K, W_V$ & Projection matrices & Linear projections for Query, Key, Value \\
$\mathbf{k}_j, \mathbf{v}_j$ & Key/Value vectors & Per-token Key and Value representations ($\in \mathbb{R}^d$) \\
$\alpha_{ij}$ & Attention weight & Softmax-normalized attention from token $i$ to $j$ \\
\midrule
\multicolumn{3}{l}{\textit{Contrastive representations}} \\
\midrule
$\mathbf{H} \in \mathbb{R}^{N \times d}$ & Neutral representations & Key/Value matrix, context-only condition \\
$\mathbf{H}^+ \in \mathbb{R}^{N \times d}$ & Positive representations & Key/Value matrix, relevant-question condition \\
$\mathbf{H}^- \in \mathbb{R}^{N \times d}$ & Negative representations & Key/Value matrix, irrelevant-question condition \\
$N$ & Number of samples & Contrastive triplets used for projection construction \\
\midrule
\multicolumn{3}{l}{\textit{Cross-covariance and SVD}} \\
\midrule
$\Omega^+$ & $\mathbf{H}^\top \mathbf{H}^+ / N$ & Uncentered cross-covariance (neutral $\times$ positive) \\
$\Omega^-$ & $\mathbf{H}^\top \mathbf{H}^- / N$ & Uncentered cross-covariance (neutral $\times$ negative) \\
$\Omega_\Delta$ & $\Omega^+ - \Omega^-$ & Differential cross-covariance \\
$U, \Sigma, V$ & SVD of $\Omega_\Delta$ & Left/right singular vectors and singular values \\
$k$ & Retained rank & Number of top singular vectors retained (set by $\gamma$) \\
\midrule
\multicolumn{3}{l}{\textit{PRISM parameters}} \\
\midrule
$P_K, P_V$ & $U_{:,:k} U_{:,:k}^\top$ & Projection matrices for Key and Value channels \\
$\mathbf{r}_i^+, \mathbf{r}_i^-$ & Head representation & Per-sample Key or Value vector, positive/negative ($\in \mathbb{R}^d$) \\
$D_{\ell,h}$ & $\frac{1}{N}\sum_i \|\mathbf{r}_i^+ - \mathbf{r}_i^-\|_2$ & Per-head discriminability score \\
$w_{\ell,h}$ & $\mathrm{softplus}(D_{\ell,h} - \delta_{\min})$ & Per-head importance weight (generic) \\
$w_{\ell,h}^K$ & $\mathrm{softplus}(D_{\ell,h}^K - \delta_{\min})$ & Per-head Key weight \\
$w_{\ell,h}^V$ & $\mathrm{softplus}(D_{\ell,h}^V - \delta_{\min})$ & Per-head Value weight \\
$\gamma$ & Variance threshold & Cumulative singular value ratio for rank selection \\
$\delta_{\min}$ & Minimum discriminability & Softplus shift parameter \\
$g_K, g_V$ & Gain parameters & Steering strength for Key and Value channels \\
\midrule
\multicolumn{3}{l}{\textit{Steering formulas}} \\
\midrule
$\mathbf{k}'_j$ & $\mathbf{k}_j + g_K \cdot w_{\ell,h}^K \cdot P_K \cdot \mathbf{k}_j$ & Steered Key vector for highlighted token $j$ \\
$\mathbf{v}'_j$ & $\mathbf{v}_j + g_V \cdot w_{\ell,h}^V \cdot P_V \cdot \mathbf{v}_j$ & Steered Value vector for highlighted token $j$ \\
\bottomrule
\end{tabular*}
\caption{Notation used throughout the paper.}
\label{tab:notation}
\end{table}

\section{Proof of Proposition~\ref{prop:discriminative}}
\label{app:proof}

\begin{proof}
\textbf{Part (a)} follows directly from the Eckart--Young--Mirsky theorem: for any matrix $M \in \mathbb{R}^{m \times n}$ with SVD $M = U \Sigma V^\top$, the rank-$k$ truncation $M_k = U[:,:k] \Sigma[:k,:k] V[:,:k]^\top$ minimizes $\|M - M_k\|_F$ over all rank-$k$ matrices. Equivalently, $U[:,:k]$ maximizes $\|U[:,:k]^\top M\|_F^2$ over all $k$-column orthonormal matrices. Applying this to $M = \Omega_\Delta$ gives part (a).

\textbf{Part (b):} If $(\Omega^+)^\top \mathbf{u}_s = (\Omega^-)^\top \mathbf{u}_s$, then $\Omega_\Delta^\top \mathbf{u}_s = (\Omega^+ - \Omega^-)^\top \mathbf{u}_s = \mathbf{0}$, so $\mathbf{u}_s$ lies in the left null space of $\Omega_\Delta$ and is orthogonal to its nonzero left singular directions.
\end{proof}

\section{Evaluation metrics}
\label{app:metrics}

We provide formal definitions of all evaluation metrics used in the three benchmarks.

\paragraph{BiasBios.}
\begin{itemize}[noitemsep,topsep=2pt,leftmargin=1em]
\item \textbf{Top-1 Accuracy.} The generated text is parsed for an occupation label. If the predicted occupation matches the ground-truth target, the sample is correct. Accuracy is the fraction of correct predictions over the test set.
\item \textbf{Fluency.} Defined as the weighted n-gram entropy of the generated text:
\begin{equation}
H_n = -\sum_g p_n(g)\log_2 p_n(g), \qquad
\text{Fluency} = \frac{1}{2}\left(\frac{2}{3}H_2+\frac{4}{3}H_3\right).
\end{equation}
where $p_n(g)$ is the empirical distribution of $n$-grams. Higher values indicate greater n-gram diversity and better fluency.
\item \textbf{Consistency.} The cosine similarity between the TF--IDF vector of the generated text and the mean TF--IDF vector of reference biographies for the target occupation, averaged over all samples.
\end{itemize}

\paragraph{CounterFact.}
\begin{itemize}[noitemsep,topsep=2pt,leftmargin=1em]
\item \textbf{Efficacy Score.} The probability that the model assigns higher likelihood to the target answer than to the original (pre-edit) answer:
\begin{equation}
\text{Efficacy} = \frac{1}{|\mathcal{D}|} \sum_{(s, o^*, o) \in \mathcal{D}} \mathbf{1}\bigl[p(o^* \mid s) > p(o \mid s)\bigr]
\end{equation}
where $s$ is the subject prompt, $o^*$ is the target (counterfactual) object, and $o$ is the original object.
\item \textbf{Paraphrase Score.} Same as Efficacy, but evaluated on paraphrased versions of the subject prompt, measuring robustness to surface-form variation.
\end{itemize}

\paragraph{Pronoun Change.}
Given an input biography containing gendered pronouns and an instruction to replace them with gender-neutral forms, two metrics are computed based on the generated text:
\begin{itemize}[noitemsep,topsep=2pt,leftmargin=1em]
\item \textbf{P.~Score} (basic pronoun conversion rate). Let $\mathcal{P}_b = \{\text{she}, \text{he}\}$ be the basic pronoun set. For each sample, count the number of basic pronouns in the original text ($n_b$) and the number still remaining in the generated text ($r_b$). The conversion rate is $(n_b - r_b) / n_b$. This is multiplied by the content overlap ratio (fraction of non-pronoun content tokens preserved) to penalize generations that trivially remove pronouns by truncating the text:
\begin{equation}
\text{P.~Score} = \frac{n_b - r_b}{n_b} \times \frac{|\text{content}_{\text{gen}} \cap \text{content}_{\text{orig}}|}{|\text{content}_{\text{orig}}|}
\end{equation}
\item \textbf{All-changed P.~Score.} Same formula but with the extended pronoun set $\mathcal{P}_a = \{\text{she}, \text{he}, \text{her}, \text{him}, \text{hers}, \text{his}, \text{herself}, \text{himself}\}$. Because $\mathcal{P}_b \subset \mathcal{P}_a$ but the denominators differ ($n_b$ vs.\ $n_a$), it is mathematically possible for All-changed to exceed P.~Score when the model converts extended forms more successfully than basic ones.
\end{itemize}

\section{Comparison with \adaseka{}}
\label{app:adaseka}

\adaseka{}~\citep{li2026spectral} extends \seka{} with a query-adaptive multi-expert routing mechanism that dynamically selects among multiple projection subspaces based on the input query's semantic intent. While this achieves strong performance on several benchmarks, it comes at substantial memory cost: +15.59\,GB peak memory, placing it in a fundamentally different memory regime from single-projection methods like \seka{} and \prism{}.

Table~\ref{tab:adaseka_full} provides the full comparison. Key observations:

\begin{table}[htb]
\centering
\renewcommand{\arraystretch}{0.9}
\small

\setlength{\tabcolsep}{3pt}
\begin{tabular*}{\textwidth}{@{\extracolsep{\fill}}ll ccccc c}
\toprule
\textbf{Benchmark} & \textbf{Model} & \seka{} & \prismk{} & \prismkv{} & & \adaseka{} & $\Delta$ \\
\midrule
\multirow{5}{*}{BiasBios}
& \qwenA{} & 90.92 & \underline{\textbf{92.38}} & 92.36 & & 91.86 & +0.52 \\
& \qwenB{} & 88.74 & \underline{\textbf{89.62}} & 89.12 & & 88.50 & +1.12 \\
& \qwenC{} & 90.28 & \underline{\textbf{91.68}} & 91.20 & & 91.22 & +0.46 \\
& \gemmaA{} & 92.42 & \textbf{92.90} & 91.64 & & \underline{92.92} & $-$0.02 \\
& \gemmaB{} & \underline{\textbf{93.04}} & 92.22 & 91.94 & & 91.14 & +1.90 \\
\midrule
\multirow{5}{*}{CounterFact}
& \qwenA{} & 99.02 & \underline{\textbf{99.14}} & 98.08 & & 98.90 & +0.24 \\
& \qwenB{} & 99.08 & \underline{\textbf{99.24}} & 98.46 & & 99.00 & +0.24 \\
& \qwenC{} & 98.92 & \underline{\textbf{99.00}} & 98.98 & & \underline{99.00} & 0.00 \\
& \gemmaA{} & 98.04 & \textbf{98.10} & 97.68 & & \underline{98.74} & $-$0.64 \\
& \gemmaB{} & \underline{\textbf{98.86}} & \underline{\textbf{98.86}} & 93.40 & & 92.48 & +6.38 \\
\midrule
\multirow{5}{*}{Pronoun}
& \qwenA{} & 95.18 & \underline{\textbf{96.18}} & 96.06 & & 94.54 & +1.64 \\
& \qwenB{} & 98.56 & 99.40 & \textbf{99.48} & & \underline{99.68} & $-$0.20 \\
& \qwenC{} & 98.66 & \textbf{99.66} & \textbf{99.66} & & \underline{99.88} & $-$0.22 \\
& \gemmaA{} & 81.53 & 89.08 & \textbf{90.16} & & \underline{93.76} & $-$3.60 \\
& \gemmaB{} & 97.70 & \underline{\textbf{97.94}} & 94.54 & & 96.88 & +1.06 \\
\bottomrule
\end{tabular*}

\caption{Full comparison including \adaseka{} (\%). \textbf{Bold}: best among lightweight methods; \underline{underline}: best overall (including \adaseka{}). $\Delta$ is the best lightweight method minus \adaseka{}.}
\label{tab:adaseka_full}
\end{table}

\begin{itemize}[noitemsep,topsep=0pt,parsep=0pt,partopsep=0pt,leftmargin=1em]
    \item \textbf{On BiasBios}, the best lightweight method outperforms \adaseka{} on 4 out of 5 models. \prismk{} leads on all three Qwen3 models, while \seka{} leads on \gemmaB{} with the largest margin of +1.90\%. \adaseka{} wins only on \gemmaA{} by 0.02\%.
    \item \textbf{On CounterFact}, \prismk{} surpasses \seka{} on 4 of 5 models and matches it on the fifth (\gemmaB{}: both 98.86\%). Both methods operate near the performance ceiling (98--99\%), so absolute differences are small.
    \item \textbf{On Pronoun Change}, \adaseka{} achieves the highest scores on 3 out of 5 models (\qwenB{}/14B, \gemmaA{}). However, on \qwenA{} and \gemmaB{}, the best lightweight method surpasses \adaseka{}.
    \item \textbf{Overall}, among the 15 model$\times$benchmark cells, the best lightweight method matches or exceeds \adaseka{} in \textbf{10 out of 15 cases}. In the reported efficiency setting, \prism{} requires 1.26$\times$ the base latency and substantially less memory than \adaseka{}.
\end{itemize}

In these 15 cells, \prism{} matches or exceeds \adaseka{} in 10 while using less memory. Whether the multi-expert overhead is acceptable depends on the deployment setting and the importance of the cells where \adaseka{} performs better.

\section{Hyperparameters}
\label{app:hyperparams}

Tables~\ref{tab:hyperparams_prismk} and~\ref{tab:hyperparams_prismkv} list the validation-selected hyperparameters for \prismk{} and \prismkv{}, respectively. Each model--benchmark configuration is selected separately by validation-set grid search.

The gain $g_K$ varies across benchmarks: ${\sim}$0.40 for BiasBios, 1.10--6.00 for CounterFact, and 0.05--0.30 for Pronoun Change, reflecting different steering intensities per task. The variance retention threshold $\gamma$ is on average higher for Qwen3 than for Gemma3. On \gemmaB{} (Pronoun Change), $g_K = -0.30$; see Section~\ref{sec:limitations} for discussion.

\begin{table}[htb]
\centering
\renewcommand{\arraystretch}{1.05}
\small

\begin{tabular*}{\textwidth}{@{\extracolsep{\fill}}llccc}
\toprule
\textbf{Model} & \textbf{Benchmark} & $\gamma$ & $\delta_{\min}$ & $g_K$ \\
\midrule
\qwenA{} & \multirow{5}{*}{BiasBios} & 0.998 & 0.08 & 0.40 \\
\qwenB{} &  & 0.998 & 0.08 & 0.40 \\
\qwenC{} &  & 0.998 & 0.08 & 0.40 \\
\gemmaA{} &  & 0.850 & 0.08 & 0.50 \\
\gemmaB{} &  & 0.980 & 0.00 & 0.40 \\
\midrule
\qwenA{} & \multirow{5}{*}{CounterFact} & 0.998 & 0.13 & 1.90 \\
\qwenB{} &  & 0.960 & 0.12 & 2.70 \\
\qwenC{} &  & 0.870 & 0.08 & 3.00 \\
\gemmaA{} &  & 0.990 & 0.12 & 6.00 \\
\gemmaB{} &  & 0.990 & 0.50 & 1.10 \\
\midrule
\qwenA{} & \multirow{5}{*}{Pronoun Change} & 0.880 & 0.15 & 0.15 \\
\qwenB{} &  & 0.900 & 0.15 & 0.05 \\
\qwenC{} &  & 0.880 & 0.15 & 0.05 \\
\gemmaA{} &  & 0.800 & 0.15 & 0.30 \\
\gemmaB{} &  & 0.700 & 0.30 & $-$0.30 \\
\bottomrule
\end{tabular*}

\caption{\prismk{} hyperparameters.}
\label{tab:hyperparams_prismk}
\end{table}

\begin{table}[htb]
\centering
\renewcommand{\arraystretch}{1.05}
\small

\begin{tabular*}{\textwidth}{@{\extracolsep{\fill}}llcccc}
\toprule
\textbf{Model} & \textbf{Benchmark} & $\gamma$ & $\delta_{\min}$ & $g_K$ & $g_V$ \\
\midrule
\qwenA{} & \multirow{5}{*}{BiasBios} & 0.998 & 0.08 & 0.40 & 0.10 \\
\qwenB{} &  & 0.998 & 0.08 & 0.40 & 0.10 \\
\qwenC{} &  & 0.998 & 0.08 & 0.40 & 0.10 \\
\gemmaA{} &  & 0.800 & 0.12 & 0.30 & 0.10 \\
\gemmaB{} &  & 0.994 & 0.00 & 0.40 & 0.10 \\
\midrule
\qwenA{} & \multirow{5}{*}{CounterFact} & 0.998 & 0.13 & 1.90 & 0.02 \\
\qwenB{} &  & 0.960 & 0.12 & 2.70 & 0.05 \\
\qwenC{} &  & 0.870 & 0.08 & 3.00 & 0.05 \\
\gemmaA{} &  & 0.990 & 0.12 & 6.00 & 0.10 \\
\gemmaB{} &  & 0.990 & 0.40 & 3.00 & 0.50 \\
\midrule
\qwenA{} & \multirow{5}{*}{Pronoun Change} & 0.880 & 0.15 & 0.15 & 0.05 \\
\qwenB{} &  & 0.900 & 0.15 & 0.05 & 0.02 \\
\qwenC{} &  & 0.880 & 0.15 & 0.05 & 0.02 \\
\gemmaA{} &  & 0.800 & 0.15 & 0.30 & 0.10 \\
\gemmaB{} &  & 0.700 & 0.30 & 0.05 & 0.02 \\
\bottomrule
\end{tabular*}

\caption{\prismkv{} hyperparameters.}
\label{tab:hyperparams_prismkv}
\end{table}

\section{Detailed experimental setup}
\label{app:exp_setup}

\paragraph{Hardware.} All experiments are conducted on NVIDIA H20 GPUs with 144\,GB memory. Projection construction uses a single GPU; evaluation uses a single GPU with batch size 256 for BiasBios and CounterFact, and batch size 1 for Pronoun Change.

\paragraph{Data splits.} For BiasBios, we evaluate on the first 5000 samples (indices 0--4999); a 500-sample validation subset is used for hyperparameter tuning. For CounterFact, samples 0--4999 serve as the validation set and samples 5000--9999 as the test set. Pronoun Change uses 500 validation and 500 test samples following \citet{li2026spectral}.

\paragraph{Evaluation protocol.} All methods use greedy decoding with a maximum of 64 new tokens for BiasBios, 32 for CounterFact, and 128 for Pronoun Change. We report Top-1 Accuracy for BiasBios, Efficacy and Paraphrase scores for CounterFact, and P.~Score and All-changed P.~Score for Pronoun Change. Fluency is measured by weighted 2- and 3-gram entropy, and Consistency by TF--IDF cosine similarity to target-occupation references.

\paragraph{Contrastive data.} The 100 synthetic contrastive QA pairs are generated by GPT-4o. Each pair contains two unrelated text contexts paired with two questions, one relevant and one irrelevant to each context. This yields 200 contrastive triplets from which Key and Value representations are extracted at the answer token position.

\paragraph{Projection construction.} For each model, we extract representations from all layers and heads under three prompting conditions (neutral, positive, negative). The differential cross-covariance $\Omega_\Delta$ is computed per head, followed by SVD. The top-$k$ singular vectors are retained based on the cumulative energy threshold $\gamma$. The mean representation norm difference determines the head's discriminability score $D_{\ell,h}$, which is mapped to a weight via softplus. The resulting projections are stored and reused across inference runs. Construction takes 3--8 minutes per model on one GPU; online steering adds one matrix--vector product per head and highlighted token.

\paragraph{Lost-in-the-Middle setup.}
\label{app:litm_setup}
We use the NaturalQuestions-based benchmark of \citet{liu2024lost} with 30 passages per input, where one gold passage contains the answer and the remaining 29 are distractors. The gold passage is placed at 7 positions (0, 4, 9, 14, 19, 24, 29) to probe positional sensitivity. Steering is applied selectively to the middle region (passages 4--25), targeting the positional recall deficit. Generation is limited to 60 tokens. We report Exact Match averaged across all 7 positions.

\section{Instruction-tuned evaluation}
\label{app:instruction_tuned}

Instruction-tuned checkpoints wrap prompts in model-specific chat templates, shifting the position at which the neutral, positive, and negative representations are read. We align extraction to the first generated assistant token for Qwen3-4B-Instruct, Qwen3-8B, and Gemma3-4B-it~\citep{yang2025qwen3,gemmateam2025gemma3}. The projection protocol is otherwise unchanged: each head uses 100 synthetic contrastive pairs, $\gamma=0.9$, held-out selection over $g_K\in[0.05,6.0]$ and $\delta_{\min}\in[0.06,0.20]$, and greedy decoding on one H20 GPU. Each cell averages five independently sampled projection subsets.

\begin{table}[htb]
\centering
\renewcommand{\arraystretch}{0.9}
\small
\setlength{\tabcolsep}{5pt}
\begin{tabular*}{\textwidth}{@{\extracolsep{\fill}}ll lrrrr}
\toprule
\textbf{Model} & \textbf{Benchmark} & \textbf{Metric} & \textbf{Vanilla} & \textbf{\seka{}} & \textbf{\prismk{}} & \textbf{\prismkv{}} \\
\midrule
Qwen3-4B-Instruct & BiasBios & Accuracy & 82.14 & 91.56 & \textbf{93.10} & 92.84 \\
Qwen3-4B-Instruct & CounterFact & Efficacy & 58.42 & 99.18 & \textbf{99.32} & 99.16 \\
Qwen3-4B-Instruct & Pronoun Change & P.~Score & 96.52 & 97.24 & 97.86 & \textbf{97.94} \\
\midrule
Qwen3-8B & BiasBios & Accuracy & 83.46 & 91.28 & 92.54 & \textbf{92.58} \\
Qwen3-8B & CounterFact & Efficacy & 54.76 & 99.22 & \textbf{99.38} & 99.12 \\
Qwen3-8B & Pronoun Change & P.~Score & 98.86 & 99.12 & 99.54 & \textbf{99.62} \\
\midrule
Gemma3-4B-it & BiasBios & Accuracy & 91.24 & 93.18 & \textbf{93.72} & 93.46 \\
Gemma3-4B-it & CounterFact & Efficacy & 68.36 & 98.42 & \textbf{98.56} & 98.47 \\
Gemma3-4B-it & Pronoun Change & P.~Score & 78.46 & 89.72 & 93.18 & \textbf{93.86} \\
\bottomrule
\end{tabular*}%

\caption{Instruction-tuned results, averaged over five projection subsets. All values are percentages.}
\label{tab:instruction_tuned}
\end{table}

As shown in Table~\ref{tab:instruction_tuned}, \prismk{} exceeds \seka{} in all nine reported cells. A one-sided cell-level sign test gives $p=0.002$, although model--benchmark cells are not independent experimental replicates. The largest relative gain is 3.9\% on Gemma3-4B-it Pronoun Change. On that rewriting benchmark, \prismkv{} exceeds \prismk{} on all three models; broader task coverage would be needed to attribute this difference generally to content-channel steering.

\section{Automatic span selection}
\label{app:automatic_span}

To evaluate \prism{} without oracle highlighting, we replace the gold span $\mathcal{S}$ by a span $\hat{\mathcal{S}}$ selected by BM25~\citep{robertson2009probabilistic}, Contriever~\citep{izacard2022contriever}, or BGE-M3~\citep{chen2024bgem3}. We report span recall $|\hat{\mathcal{S}}\cap\mathcal{S}^{\star}|/|\mathcal{S}^{\star}|$ on BiasBios and Lost-in-the-Middle. Projection $P_K$, head weights $w_{\ell,h}^K$, and gain $g_K$ are held fixed across selectors, so only the highlighted span changes. Each row averages three seeds.

\begin{table}[htb]
\centering
\renewcommand{\arraystretch}{0.9}
\small
\begin{tabular*}{\textwidth}{@{\extracolsep{\fill}}llrrrr}
\toprule
\textbf{Benchmark} & \textbf{Selector} & \textbf{Recall} & \textbf{Vanilla} & \textbf{\prismk{}} & \textbf{Gain} \\
\midrule
\multirow{4}{*}{BiasBios} & BM25 & 0.58 & 79.80 & 86.42 & +6.62 \\
& Contriever & 0.68 & 79.80 & 89.14 & +9.34 \\
& BGE-M3 & 0.76 & 79.80 & 90.86 & +11.06 \\
& Oracle span & 1.00 & 79.80 & 92.38 & +12.58 \\
\midrule
\multirow{4}{*}{Lost-in-the-Middle} & BM25 & 0.48 & 45.47 & 47.86 & +2.39 \\
& Contriever & 0.62 & 45.47 & 49.42 & +3.95 \\
& BGE-M3 & 0.72 & 45.47 & 50.28 & +4.81 \\
& Oracle span & 1.00 & 45.47 & 51.43 & +5.96 \\
\bottomrule
\end{tabular*}
\caption{End-to-end accuracy with automatic span selection on \qwenA{}, averaged over three seeds.}
\label{tab:automatic_span}
\end{table}

For the three tested selectors (Table~\ref{tab:automatic_span}), higher observed span recall coincides with higher accuracy, and every reported result remains above Vanilla. BGE-M3 recovers 88\% of the oracle-span gain on BiasBios and 81\% on Lost-in-the-Middle. Because selector identity, selected content, and recall change together, this comparison does not isolate recall as the causal variable.

\section{Qualitative evaluation and extended channel decomposition}
\label{app:qualitative}

We use Claude Opus 4.6 as a zero-shot judge to rate factuality (faithfulness to highlighted information) and coherence (fluency and grammaticality) on a 1--5 Likert scale. On the pooled set of 50 human-annotated samples used for this check, judge and human scores have Pearson correlation $r=0.87$. This limited validation does not establish reliability for other tasks, judges, or scoring criteria.

\begin{table}[htb]
\centering
\renewcommand{\arraystretch}{0.9}
\small
\begin{tabular*}{\textwidth}{@{\extracolsep{\fill}}lrrrr}
\toprule
& \multicolumn{2}{c}{\textbf{BiasBios / \qwenA{}}} & \multicolumn{2}{c}{\textbf{Pronoun Change / \gemmaA{}}} \\
\cmidrule(lr){2-3}\cmidrule(lr){4-5}
\textbf{Method} & \textbf{Factuality} & \textbf{Coherence} & \textbf{Factuality} & \textbf{Coherence} \\
\midrule
Vanilla & 3.12 & \textbf{4.28} & 2.45 & \textbf{4.15} \\
\seka{} & 3.68 & 3.15 & 3.82 & 3.22 \\
\prismk{} & \textbf{3.82} & 3.78 & 4.25 & 3.68 \\
\prismkv{} & 3.81 & 4.05 & \textbf{4.32} & 3.85 \\
\bottomrule
\end{tabular*}
\caption{LLM-as-Judge evaluation on a 1--5 scale. On the pooled 50-sample validation set, judge and human scores have Pearson $r=0.87$.}
\label{tab:qualitative_judge}
\end{table}

In the two settings in Table~\ref{tab:qualitative_judge}, \prismk{} improves the judge's factuality score over \seka{}, while \prismkv{} moves coherence toward Vanilla. Table~\ref{tab:dual_channel_extended} separately reports metric-based channel decompositions for five selected model--task settings.

\begin{table}[htb]
\centering
\renewcommand{\arraystretch}{0.9}
\small
\setlength{\tabcolsep}{2pt}
\begin{tabular*}{\textwidth}{@{\extracolsep{\fill}}lllrrrrr}
\toprule
\textbf{Model} & \textbf{Task} & \textbf{Metric} & \textbf{Vanilla} & \textbf{\seka{}} & \textbf{\prismk{}} & \textbf{\prism{}-V} & \textbf{\prismkv{}} \\
\midrule
\multirow{3}{*}{\qwenA{}} & \multirow{3}{*}{BiasBios} & Accuracy & 79.800 & 90.920 & 92.380 & 82.440 & 92.360 \\
& & Fluency & 4.638 & 3.681 & 4.134 & 4.666 & 4.116 \\
& & Fluency cost & 0.000 & 0.957 & 0.504 & $-$0.028 & 0.522 \\
\midrule
\multirow{3}{*}{\gemmaA{}} & \multirow{3}{*}{BiasBios} & Accuracy & 89.900 & 92.420 & 92.900 & 90.120 & 91.640 \\
& & Fluency & 4.850 & 3.920 & 4.410 & 4.902 & 4.398 \\
& & Fluency cost & 0.000 & 0.930 & 0.440 & $-$0.052 & 0.452 \\
\midrule
\multirow{3}{*}{\qwenC{}} & \multirow{3}{*}{BiasBios} & Accuracy & 85.100 & 90.280 & 91.680 & 88.920 & 91.200 \\
& & Fluency & 4.712 & 3.854 & 4.318 & 4.736 & 4.302 \\
& & Fluency cost & 0.000 & 0.858 & 0.394 & $-$0.024 & 0.410 \\
\midrule
\multirow{3}{*}{\qwenA{}} & \multirow{3}{*}{Pronoun Change} & P.~Score & 93.140 & 95.180 & 96.180 & 94.020 & 96.060 \\
& & Fluency & 3.635 & 2.956 & 3.210 & 3.682 & 3.198 \\
& & Fluency cost & 0.000 & 0.679 & 0.425 & $-$0.047 & 0.437 \\
\midrule
\multirow{3}{*}{\gemmaA{}} & \multirow{3}{*}{Pronoun Change} & P.~Score & 41.340 & 81.530 & 89.080 & 52.360 & 90.160 \\
& & Fluency & 3.457 & 2.845 & 3.102 & 3.512 & 3.089 \\
& & Fluency cost & 0.000 & 0.612 & 0.355 & $-$0.055 & 0.368 \\
\bottomrule
\end{tabular*}%

\caption{Extended dual-channel decomposition. Fluency cost is Vanilla fluency minus method fluency.}
\label{tab:dual_channel_extended}
\end{table}

Value-only steering matches or exceeds Vanilla fluency in these five settings, while \prismk{} incurs 46--63\% of \seka{}'s fluency cost. On \gemmaA{} Pronoun Change, \prismkv{} exceeds \prismk{} by 1.08 points. This isolated result is compatible with a contribution from the cross term in Eq.~\ref{eq:decomposition}, but broader interaction studies would be needed to establish when it occurs.

\section{Cross-domain projection transfer}
\label{app:cross_domain}

We transfer projections bidirectionally between BiasBios and CounterFact while retaining the target domain's optimal gain. No projection reconstruction or retraining is performed.

\begin{table}[htb]
\centering
\renewcommand{\arraystretch}{0.9}
\small
\begin{tabular*}{\textwidth}{@{\extracolsep{\fill}}llrrrr}
\toprule
\textbf{Model} & \textbf{Target} & \textbf{Vanilla} & \textbf{Task-specific} & \textbf{Transfer} & \textbf{Degradation} \\
\midrule
\multirow{2}{*}{\qwenA{}} & CounterFact & 45.00 & 99.14 & 94.86 & $-$4.28 \\
 & BiasBios & 79.80 & 92.38 & 87.42 & $-$4.96 \\
\midrule
\multirow{2}{*}{\gemmaA{}} & CounterFact & 55.04 & 98.10 & 91.24 & $-$6.86 \\
 & BiasBios & 89.90 & 92.90 & 91.08 & $-$1.82 \\
\bottomrule
\end{tabular*}
\caption{Cross-domain projection transfer. The target-domain gain is retained without rebuilding the projection. Degradation is measured from the task-specific projection.}
\label{tab:cross_domain_transfer}
\end{table}

\begin{table}[htb]
\centering
\renewcommand{\arraystretch}{0.9}
\small
\begin{tabular*}{\textwidth}{@{\extracolsep{\fill}}lrrr}
\toprule
\textbf{Model} & \textbf{BiasBios$\rightarrow$CounterFact} & \textbf{CounterFact$\rightarrow$BiasBios} & \textbf{Average} \\
\midrule
\qwenA{} & 0.64 & 0.60 & 0.62 \\
\gemmaA{} & 0.64 & 0.78 & 0.71 \\
\bottomrule
\end{tabular*}
\caption{Frobenius similarity between source- and target-domain projection subspaces.}
\label{tab:cross_domain_similarity}
\end{table}

Transferred projections remain above Vanilla in the four tested settings (Table~\ref{tab:cross_domain_transfer}), with degradation of 1.82--6.86 points relative to task-specific projections. Descriptively, the highest measured similarity (0.78, Gemma3-4B CounterFact$\rightarrow$BiasBios; Table~\ref{tab:cross_domain_similarity}) coincides with the smallest degradation. With only four transfer settings, these values do not establish a stable correlation or causal explanation. They show that projection reuse can retain gains in the tested cases, while target-domain gain selection and task-specific construction remain necessary for the reported peak accuracy.

\section{Robustness to imperfect highlights}
\label{app:highlight_noise}

We replace 5--20\% of tokens in the oracle highlighted span with irrelevant background tokens on \qwenA{} BiasBios, using three seeds per noise level.

\begin{table}[htb]
\centering
\renewcommand{\arraystretch}{0.9}
\small
\begin{tabular*}{\textwidth}{@{\extracolsep{\fill}}lll}
\toprule
\textbf{Noise ratio} & \textbf{Accuracy} & $\boldsymbol{\Delta}$~$\downarrow$ \\
\midrule
0\% & 92.38 & 0.00 \\
5\% & 91.95 & $-$0.43 \\
10\% & 91.42 & $-$0.96 \\
15\% & 90.68 & $-$1.70 \\
20\% & 89.75 & $-$2.63 \\
\bottomrule
\end{tabular*}
\caption{Robustness to irrelevant tokens substituted into highlighted spans on \qwenA{} BiasBios, averaged over three seeds. $\Delta$ is the accuracy change relative to the noise-free setting.}
\label{tab:noise_robustness}
\end{table}

In this \qwenA{} BiasBios perturbation (Table~\ref{tab:noise_robustness}), 15\% noise reduces accuracy by 1.70 points and 20\% noise by 2.63 points. Appendix~\ref{app:automatic_span} evaluates a separate source of upstream error using three concrete selectors; neither experiment covers all possible span-boundary, omission, or retrieval errors.

\section{Cross-model K/V complementarity}
\label{app:kv_cross_model}

\begin{table}[htb]
\centering
\renewcommand{\arraystretch}{0.9}
\small
\begin{tabular*}{\textwidth}{@{\extracolsep{\fill}}lccccc}
\toprule
\textbf{Model} & \textbf{Layers} & \textbf{Heads} & \textbf{K mean} & \textbf{V mean} & \textbf{K/V ratio} \\
\midrule
\qwenA{} & 36 & 8 & 0.137 & 0.135 & 1.02 \\
\qwenB{} & 36 & 8 & 0.129 & 0.174 & 0.74 \\
\qwenC{} & 40 & 8 & 0.138 & 0.237 & 0.58 \\
\gemmaA{} & 34 & 4 & 0.451 & 0.338 & 1.33 \\
\gemmaB{} & 48 & 8 & 0.509 & 0.420 & 1.21 \\
\bottomrule
\end{tabular*}
\caption{Key and Value discriminative signal strength across models.}
\label{tab:kv_cross_model}
\end{table}

Table~\ref{tab:kv_cross_model} and Figure~\ref{fig:kv_cross_model} report Key and Value discriminative signal on BiasBios for five models. Both channels have nonzero measured signal, but their balance differs.

\paragraph{Qwen3 family.} In these BiasBios measurements, the K/V ratio decreases with model size (1.02 $\to$ 0.74 $\to$ 0.58). On \qwenC{}, the average Value signal is larger than the Key signal, and Figure~\ref{fig:kv_cross_model} shows Value-dominant late layers.

\paragraph{Gemma3 family.} Both measured models have larger absolute norm differences than the Qwen3 models (K mean: 0.45--0.51 vs.\ 0.13--0.14) and K/V ratios above 1.2 on BiasBios. These descriptive statistics do not by themselves determine the optimal steering gains.

\begin{figure}[htb]
    \centering
    \includegraphics[width=\textwidth]{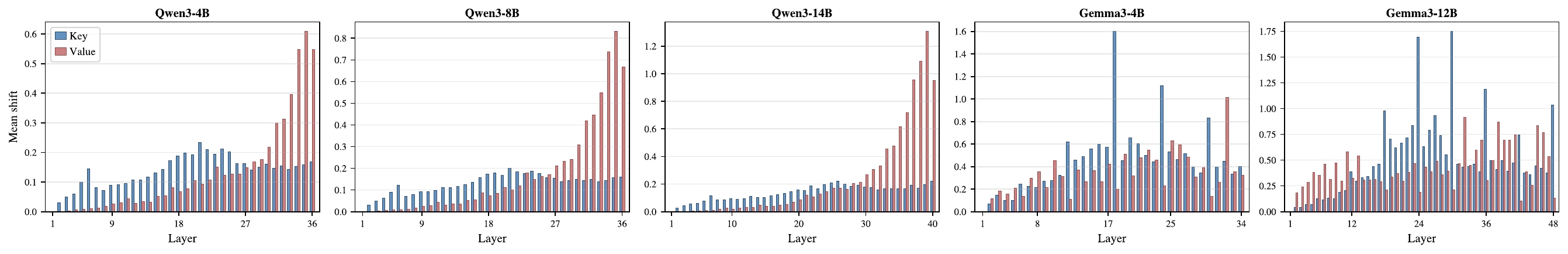}
\caption{Layer-wise Key and Value discriminative signal on BiasBios for five models. In these measurements, Qwen3 models show increasing Value dominance in late layers, while Gemma3 models remain Key-dominant.}
    \label{fig:kv_cross_model}
\end{figure}

\section{Projection stability across random seeds}
\label{app:seed_sensitivity}

We construct five projections using different random subsets of 70 from 100 available QA pairs. The subsets overlap by approximately 70\%, so this experiment measures sensitivity to partial data replacement rather than fully independent training sets. All other hyperparameters are fixed ($\gamma{=}0.998$, $\delta_{\min}{=}0.08$, $g_K{=}0.40$).

\begin{table}[htb]
\centering
\small
\begin{tabular*}{\textwidth}{@{\extracolsep{\fill}}ccc}
\toprule
\textbf{Seed} & \textbf{\qwenA{}} & \textbf{\qwenB{}} \\
\midrule
1 & 92.22 & 89.08 \\
2 & 92.12 & 89.34 \\
3 & 92.08 & 88.98 \\
4 & 92.14 & 89.30 \\
5 & 92.12 & 89.34 \\
\midrule
\textbf{Mean $\pm$ Std} & \textbf{92.14 $\pm$ 0.05} & \textbf{89.21 $\pm$ 0.15} \\
\bottomrule
\end{tabular*}
\caption{Projection sensitivity across overlapping random subsets (BiasBios, Top-1 Accuracy \%).}
\label{tab:seed_sensitivity}
\end{table}

Across these five overlapping subsets (Table~\ref{tab:seed_sensitivity}), the standard deviation is 0.05\% on \qwenA{} and 0.15\% on \qwenB{}. Both are smaller than the corresponding \prismk{}--\seka{} gaps in this BiasBios experiment (+1.46 and +0.88 points). This check does not quantify projection variability for the other models or benchmarks.

\section{Statistical reliability details}
\label{app:statistics}

All reported evaluations use greedy decoding, so generation is deterministic for a fixed projection and configuration. Appendix~\ref{app:seed_sensitivity} isolates one source of variability---partial replacement of projection data---on BiasBios with \qwenA{} and \qwenB{}. It does not provide an upper bound for every model--benchmark cell or for hyperparameter-selection uncertainty.

Across the 15 base-model cells in Table~\ref{tab:main_results}, \prismk{} matches or exceeds \seka{} in 14 (one-sided sign test after excluding the tie: $p < 0.001$). It is strictly higher on all five Pronoun Change P.~Score cells ($p = 0.031$) and on four of five CounterFact Efficacy cells, with one tie. These cell-level sign tests summarize directional consistency; model and task cells are not statistically independent replicates.

\section{CounterFact ablation}
\label{app:cf_ablation}

\begin{table}[htb]
\centering
\renewcommand{\arraystretch}{0.9}
\small
\begin{tabular*}{\textwidth}{@{\extracolsep{\fill}}llcc}
\toprule
\textbf{Configuration} & \textbf{Projection} & \textbf{Weighting} & \textbf{Efficacy (\%)} \\
\midrule
\prismk{} & Differential & Softplus & \textbf{99.14} \\
Indep.\ + softplus & Independent & Softplus & 98.84 \\
Diff.\ + uniform & Differential & Uniform & 98.62 \\
\bottomrule
\end{tabular*}
\caption{Ablation on CounterFact with \qwenA{} (Efficacy \%, test set).}
\label{tab:cf_ablation}
\end{table}

Differential projection contributes +0.30\% (Table~\ref{tab:cf_ablation}), computed as the gap between \prismk{} at 99.14\% and the independent-projection variant at 98.84\%. Softplus weighting contributes +0.52\%, measured as the gap between \prismk{} and the uniform-weight variant at 98.62\%. These are qualitatively consistent with the BiasBios ablation in Table~\ref{tab:ablation}.

The smaller absolute contributions on this \qwenA{} CounterFact setting occur near the performance ceiling. The component ordering matches the \qwenA{} BiasBios ablation, but one additional model--task setting is insufficient to establish generality.
\section{Data quantity ablation}
\label{app:data_ablation}

\begin{table}[htb]
\centering
\renewcommand{\arraystretch}{0.9}
\small

\begin{tabular*}{\textwidth}{@{\extracolsep{\fill}}cc}
\toprule
\textbf{Synthetic samples} & \textbf{Top-1 Accuracy (\%)} \\
\midrule
50 & 91.92 \\
100 (default) & 92.38 \\
200 & 92.40 \\
\bottomrule
\end{tabular*}

\caption{Effect of synthetic data quantity on \prismk{} for \qwenA{} on BiasBios.}
\label{tab:data_ablation}
\end{table}

On \qwenA{} BiasBios (Table~\ref{tab:data_ablation}), 50 samples reach 91.92\%, 100 reach 92.38\%, and 200 reach 92.40\%. This single model--task comparison motivates the default of 100 samples for the reported setup but does not establish the same saturation point elsewhere.

\section{Direction consistency}
\label{app:direction_consistency}

\begin{table}[htb]
\centering
\renewcommand{\arraystretch}{0.9}
\small

\setlength{\tabcolsep}{4pt}
\begin{tabular*}{\textwidth}{@{\extracolsep{\fill}}lccc}
\toprule
& $\Omega^+$ (independent) & $\Omega_\Delta$ (differential) & Random baseline \\
\midrule
Cross-head similarity (mean) & 0.254 & \textbf{0.068} & 0.071 \\
Cross-head similarity (early layers) & 0.353 & \textbf{0.053} & 0.071 \\
\bottomrule
\end{tabular*}

\caption{Cross-head cosine similarity of top singular vectors on \qwenA{}.}
\label{tab:direction}
\end{table}

To quantify whether different heads learn independent or redundant directions, we compute the mean pairwise absolute cosine similarity between the top left singular vectors of all heads in \qwenA{}. A random baseline is computed by averaging cosine similarities between uniformly sampled unit vectors in $\mathbb{R}^{128}$, yielding 0.071.

In this \qwenA{} analysis, $\Omega^+$ directions have cross-head similarity 0.254, compared with 0.071 for the random baseline, while $\Omega_\Delta$ directions reach 0.068. This pattern is consistent with differential decomposition removing directions shared across heads; cosine similarity alone does not identify the semantic content of those directions.

In early layers, the measured values are 0.353 for $\Omega^+$ and 0.053 for $\Omega_\Delta$. Whether this pattern transfers to other models and tasks remains to be tested.

\section{Hyperparameter sensitivity}
\label{app:sensitivity}

We sweep each hyperparameter independently while holding others at their optimal values (\qwenA{}, BiasBios).

\paragraph{$g_K$ sensitivity.} On \qwenA{} BiasBios, performance varies by less than 0.5 points for tested $g_K \in [0.38, 0.50]$ and declines outside that interval.

\paragraph{$\delta_{\min}$ sensitivity.} For the same setting, performance varies by less than 0.45 points over the tested interval $\delta_{\min} \in [0.06, 0.10]$.

\paragraph{$\gamma$ sensitivity.} For the same setting, performance varies by less than 0.1 points over the tested interval $\gamma \in [0.990, 0.998]$.

\section{Head weight and per-sample analysis}
\label{app:head_logprob}

\paragraph{Head weight distribution.} Figure~\ref{fig:head_weights} compares softplus weighting with hard thresholding. Hard thresholding shuts off 108 heads entirely, while softplus assigns weak-signal heads reduced but nonzero weights.

\begin{figure}[htb]
    \centering
    \includegraphics[width=0.75\textwidth]{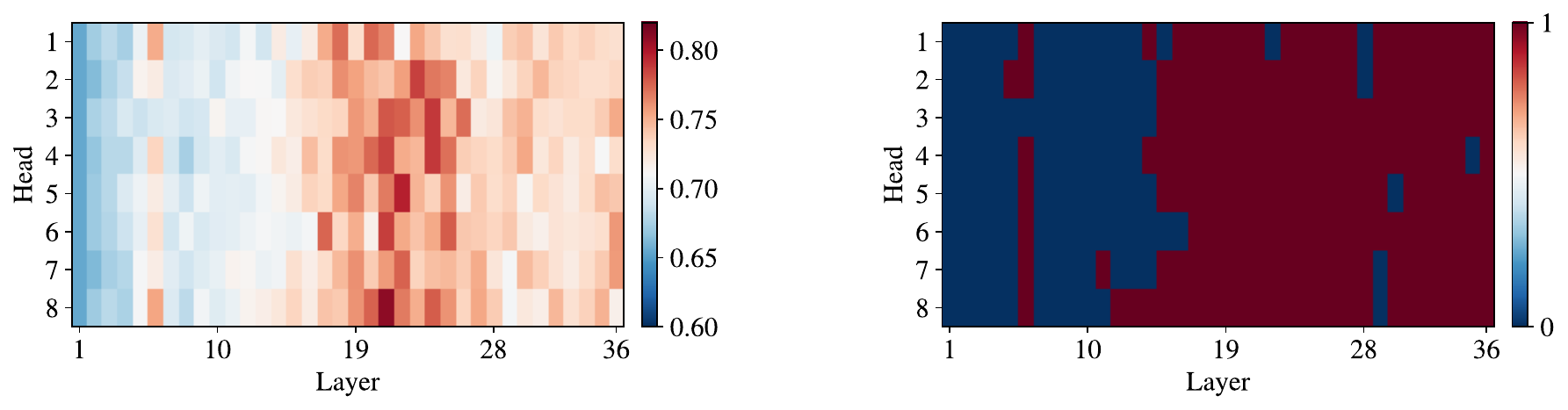}
    \caption{Head weight heatmaps. \textbf{Left:} \prism{} softplus weights vary continuously across heads. \textbf{Right:} \seka{} hard thresholding deactivates 108 heads, including 90\% of early-layer heads.}
    \label{fig:head_weights}
\end{figure}

\begin{figure}[htb]
    \centering
    \includegraphics[width=0.42\textwidth]{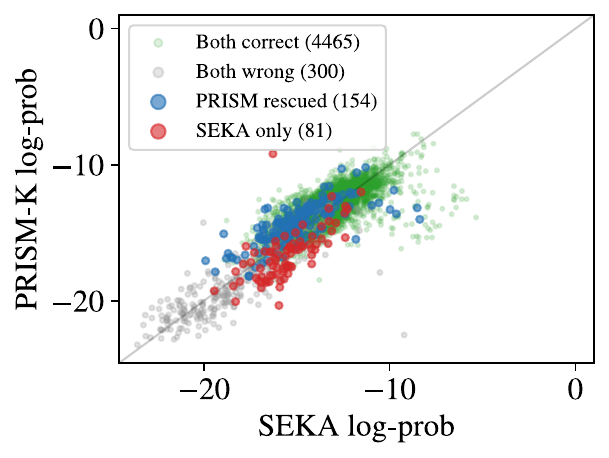}
    \caption{Per-sample target log-probability comparison on BiasBios with \qwenA{}. \prismk{} rescues 154 samples that \seka{} misses and loses 81, for a net gain of 73.}
    \label{fig:logprob}
\end{figure}

\section{Case study}
\label{app:case_study}

We identify 154 samples (33.9\% of \seka{}'s 454 errors on BiasBios with \qwenA{}) where \seka{} predicts incorrectly but \prismk{} succeeds. Figure~\ref{fig:logprob} reveals that \prismk{} rescues 154 samples while losing only 81, for a net gain of 73 samples (+1.46\% accuracy). Representative examples:

\paragraph{Success cases.} The following examples illustrate contexts where \seka{} fails but \prismk{} succeeds:

\begin{itemize}[noitemsep,topsep=3pt,parsep=2pt,partopsep=0pt,leftmargin=1em]
    \item \emph{``PhD, Adjunct Assistant Professor in Psychiatry...''} $\to$ \seka{}: psychologist; \prismk{}: \textbf{professor} \checkmark
    \item \emph{``assistant professor of psychology at NUS...''} $\to$ \seka{}: psychologist; \prismk{}: \textbf{professor} \checkmark
    \item \emph{``Performance Architect at Cisco...''} $\to$ \seka{}: software engineer; \prismk{}: \textbf{architect} \checkmark
    \item \emph{``aspiring filmmaker...shoots promotional videos...''} $\to$ \seka{}: photographer; \prismk{}: \textbf{filmmaker} \checkmark
    \item \emph{``is a Maine comedian...''} $\to$ \seka{}: photographer; \prismk{}: \textbf{comedian} \checkmark
\end{itemize}

In many rescued examples, the occupation keyword co-occurs with a related domain word or company name. For instance, ``Psychiatry'' and ``psychology'' accompany cases where \seka{} predicts ``psychologist'' despite the biography describing a professor; ``Cisco'' accompanies a ``software engineer'' prediction despite ``Architect'' appearing explicitly. These examples are compatible with the shared-feature account motivating differential projection, but they do not directly identify the internal cause of either prediction.

\paragraph{Failure analysis.} Among the 81 \qwenA{} BiasBios samples where \prismk{} is wrong but \seka{} succeeds, the biographies tend to contain short or ambiguous occupation descriptions. In this paired comparison, \prismk{} rescues 154 samples and loses 81 relative to \seka{}; this count does not characterize failure modes on other models or tasks.

\section{LLM Usage}
\label{app:llm_usage}

GPT-4o was used to generate the 100 synthetic contrastive QA pairs for projection construction (Appendix~\ref{app:exp_setup}). Claude Opus 4.6 (\texttt{claude-opus-4-6}) was used as the qualitative judge reported in Appendix~\ref{app:qualitative}; its ratings were validated against human annotations. LLM-based tools were also used for language polishing and copy editing. The authors verified all experimental results, citations, and scientific claims.

\end{document}